\documentclass{article}

\usepackage[preprint]{neurips_2026}

\usepackage[utf8]{inputenc} 
\usepackage[T1]{fontenc}    
\usepackage{hyperref}       
\usepackage{url}            
\usepackage{booktabs}       
\usepackage{amsfonts}       
\usepackage{nicefrac}       
\usepackage{microtype}      
\usepackage{xcolor}         

\usepackage{amsmath}
\usepackage{amsthm}
\newtheorem{lemma}{Lemma}[section]
\usepackage{subcaption}
\usepackage{graphicx}
\usepackage[svgnames]{xcolor}
\usepackage[table]{xcolor}
\usepackage{enumitem}
\usepackage{multirow}
\usepackage{tcolorbox}
\usepackage{url}
\usepackage{hyperref}
\usepackage{pifont}

\definecolor{tickgreen}{HTML}{59A14F}
\newcommand{\cmark}{\textcolor{tickgreen}{\ding{51}}}
\definecolor{xred}{HTML}{e15759}
\newcommand{\xmark}{\textcolor{xred}{\ding{55}}}
\usepackage[normalem]{ulem}

\title{DiRotQ: Rotation-Aware Quantization for 4-bit Diffusion Transformers}

\author{%
    Sayeh Sharify\thanks{Work done while at d-Matrix.} \\
    d-Matrix \\
    Santa Clara, CA 95054 \\
    sayeh@google.com \\
    \And
    Mahsa Salmani \\
    d-Matrix \\
    Santa Clara, CA 95054 \\
    msalmani@d-matrix.ai \\
    \And
    Hesham Mostafa \\
    d-Matrix \\
    Santa Clara, CA 95054 \\
    hmostafa@d-matrix.ai \\
}

\begin{document}

\maketitle

\begin{abstract}
\hyphenpenalty=10000
\exhyphenpenalty=10000
Diffusion Transformers achieve state-of-the-art image generation quality but incur substantial memory and computational costs at inference. While aggressive \mbox{Post-Training Quantization (PTQ)} to 4-bit precision offers significant efficiency gains, it often causes severe quality degradation. Existing approaches, including smoothing, mixed-precision, rotation, and low-rank residual methods, mitigate this degradation but still leave a noticeable gap to FP16/BF16 performance. We introduce \emph{DiRotQ}, a W4A4 PTQ framework that addresses this gap through rotation-aware activation quantization. DiRotQ identifies a low-rank subspace capturing dominant activation variance via Principal Component Analysis (PCA), preserving this subspace at higher precision while quantizing the remaining components to 4-bit. At inference, activations are rotated into the PCA basis using calibration-derived orthogonal transformations, with the inverse rotation fused into the layer weights offline. Combined with GPTQ-based weight quantization, DiRotQ achieves an FID $(\downarrow)$ of $15.9$ and PSNR $(\uparrow)$ of $19.1$~dB on PixArt-$\Sigma$ over MJHQ-30K, outperforming SVDQuant \mbox{(FID $18.9$, PSNR $17.6$)} under the same INT W4A4 setting. We further adopt a \emph{VLM-as-a-Judge} evaluation protocol for diffusion model quantization, providing a holistic assessment of perceptual quality and prompt alignment under aggressive compression. On the systems side, our \emph{Triton-based custom kernel} enables efficient end-to-end inference, reducing memory usage of the 12B FLUX.1-dev model by $2.1\times$ and achieving a $2.3\times$ speedup over BF16 on an RTX 4090. On the RTX PRO 6000 with Blackwell architecture, DiRotQ achieves a $2.0\times$ speedup and $2.0\times$ memory reduction over BF16 using NVFP4. Our codebase is available \href{https://github.com/sayehsharify/DiRotQ}{here}.
\end{abstract}

\section{Introduction}
\label{sec:introduction}

Diffusion models~\citep{sohl2015deep} have become the dominant generative framework for visual, audio, and multimodal generation, achieving state-of-the-art performance in text-to-image~\citep{rombach2022high, saharia2022photorealistic, balaji2022ediff}, text-driven image editing~\citep{avrahami2022blended}, image synthesis~\citep{dhariwal2021diffusion, podell2023sdxl}, and audio generation~\citep{kong2020diffwave}. Despite their strong performance, diffusion models are computationally and memory intensive, requiring iterative denoising over many timesteps, with each step incurring additional cost from the quadratic complexity of self-attention over image patch tokens. Post-Training Quantization (PTQ) reduces memory footprint and inference latency without retraining~\citep{shang2023post, he2023ptqd, li2025svdquant}.
As shown in prior work~\citep{li2025svdquant}, diffusion models remain compute-bound even at small batch sizes, making weight-only quantization insufficient for speedup. Both weights and activations must use the same precision; otherwise, weights are upcast during computation, negating any performance gains.
While prior work has explored low-bit quantization, a gap to full-precision baselines persists under aggressive settings. This work focuses on joint quantization of activations and weights to enable efficient low-bit inference while narrowing this gap.

\begin{figure*}[t!]
    \centering
    \includegraphics[scale=0.615]{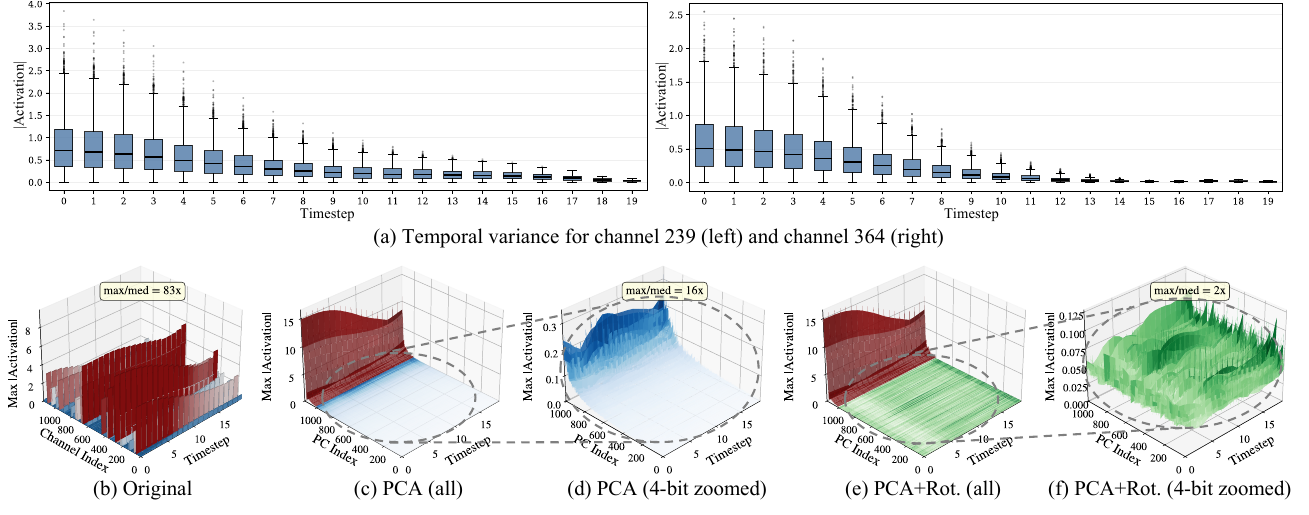}
     \vspace*{-1mm}
     
    \caption{PixArt-$\Sigma$~\citep{chen2024pixart} Block 2, self-attention input activation distributions across $20$ denoising timesteps. \textbf{(a)}~Temporal variance for channels 239 and 364. \textbf{(b)}~Original activations show severe channel-wise variance with large outliers ($\nicefrac{\max}{\text{median}}=83\times$). \textbf{(c, d)}~PCA isolates outliers in a \textcolor{DarkRed}{$16$-bit} subspace (\textcolor{DarkRed}{red}), reducing the max-to-median ratio in the remaining \textcolor{SteelBlue}{$4$-bit} channels (\textcolor{SteelBlue}{blue}) to $16\times$. \mbox{\textbf{(e, f)}}~Rotation further equalizes the \textcolor{SeaGreen}{$4$-bit} subspace (\textcolor{SeaGreen}{green}), bringing this ratio down to $2\times$.}
    \label{fig:model_inference}
    
    \vspace*{-5mm}
\end{figure*}

We propose \textit{DiRotQ}, a PCA-based rotation-aware quantization framework that aligns the quantization basis with the activation distribution to enable accurate W4A4 inference. As shown in Figure~\ref{fig:model_inference}, raw Diffusion Transformer (DiT) activations exhibit pronounced timestep- (Figure~\ref{fig:model_inference}a) and channel-dependent (Figure~\ref{fig:model_inference}b) outliers, making low-bit quantization highly error-prone. 
With only 16 quantization levels available at 4-bit precision, the shared scale for each group is determined by its largest-magnitude values. Consequently, a few outlier channels can dominate the scaling, leaving fewer effective quantization levels for the remaining values and increasing quantization error. 
SVDQuant~\citep{li2025svdquant} addresses this by shifting activation outliers into the weights and applying low-rank Singular Value Decomposition (SVD); however, because it operates in weight space, SVDQuant only partially captures activation statistics. 
In contrast, DiRotQ operates directly in activation space by rotating activations into their PCA basis, concentrating outlier energy within a small high-variance subspace that is retained at high precision. This transformation makes the remaining components substantially more uniform across channels and timesteps, reducing the max-to-median ratio from approximately $83\times$ to $16\times$. A subsequent orthogonal rotation further equalizes the residual components, reducing the ratio to approximately $2\times$ and enabling efficient 4-bit quantization (Figure~\ref{fig:model_inference}f).
This improved conditioning translates into higher quantization fidelity: as shown in Figure~\ref{fig:sqnr_sa_out}, for the \textit{self-attention output projection layer}, \textit{Activation PCA} (the criterion used by DiRotQ) improves activation Signal-to-Quantization-Noise Ratio (SQNR) by $4$\texttt{–}$10$~dB over \textit{Weight SVD} (the criterion used by SVDQuant) and by $6$\texttt{–}$12$~dB over Round-to-Nearest \textit{(RTN)} INT4 quantization across transformer blocks and denoising timesteps. Additional layer-wise results are provided in Appendix~\ref{sec:act-sqnr-appendix}, Figure~\ref{fig:sqnr-appendix}.


Although rotation-based quantization has been highly effective for LLMs (e.g., QuaRot~\citep{ashkboos2024quarot}, ResQ~\citep{saxena2024resq}, SpinQuant~\citep{liu2025spinquant}), prior work~\citep{li2025svdquant} argued it is inapplicable to diffusion models. First, due to adaptive normalization layers (e.g., AdaLN~\citep{peebles2023scalable}), rotation cannot be absorbed into preceding weights, requiring explicit handling during inference. Second, diffusion models exhibit strong timestep-dependent variation in activation magnitudes, making naive cross-layer sharing of rotation matrices ineffective. DiRotQ addresses both challenges with a tailored design: rotations are computed per layer to respect layer-specific statistics, and a single time-averaged PCA basis per layer is shared across timesteps, with timestep-dependent magnitude shifts absorbed by per-token activation scales and a small high-precision tail. This design keeps only the forward rotation online while fusing its inverse into the static weights, making rotation-based W4A4 quantization practical and effective for DiTs.


We evaluate DiRotQ on three state-of-the-art DiT models, PixArt-$\Sigma$~\citep{chen2024pixart}, FLUX.1-dev~\citep{flux2024}, and SANA-1.6B~\citep{xie2025sana}, across two diverse datasets (MJHQ-30K~\citep{li2024playground} and sDCI~\citep{urbanek2024picture}), reporting quality and similarity metrics. To better assess the impact of quantization on generation quality, we further adopt a \emph{Vision-Language Model (VLM)-as-a-Judge evaluation protocol}, which has recently emerged as a reliable alternative for assessing text-to-image models~\citep{chen2024mj, hayes2025finegrain, chen2024mllm, lee2024prometheus}. Unlike traditional methods, VLM-based judges jointly reason about visual quality and prompt alignment, making them well-suited for capturing subtle degradation introduced by low-bit quantization. To our knowledge, this is the first work to incorporate VLM-based evaluation for diffusion model quantization, enabling a more holistic assessment of generation quality under aggressive compression. We make the following contributions:
\begin{enumerate}[nosep, wide=0pt, leftmargin=*]
\item We propose DiRotQ, a W4A4 post-training quantization framework for DiTs that uses PCA-based activation rotation to decompose activations into a high-variance subspace retained in higher precision and a low-variance subspace quantized to 4-bit, substantially reducing quantization error.
\item We adopt a VLM-as-a-Judge evaluation protocol for diffusion model quantization, validated across multiple independent open- and closed-source judges and complementary scoring axes, to provide a more holistic and robust assessment of perceptual quality and prompt alignment.
\item We implement a Triton-based~\citep{tillet2019triton} custom kernel for efficient rotation and achieve runtime speedups and memory reductions on NVIDIA RTX 4090~\citep{nvidia_rtx4090} and RTX PRO 6000~\citep{nvidia_rtx6000pro} GPUs with our W4A4 quantized models.
\end{enumerate}

\section{Related work}
\label{sec:related_work}
\begin{figure*}[t!]
    \centering
    \includegraphics[scale=0.64]{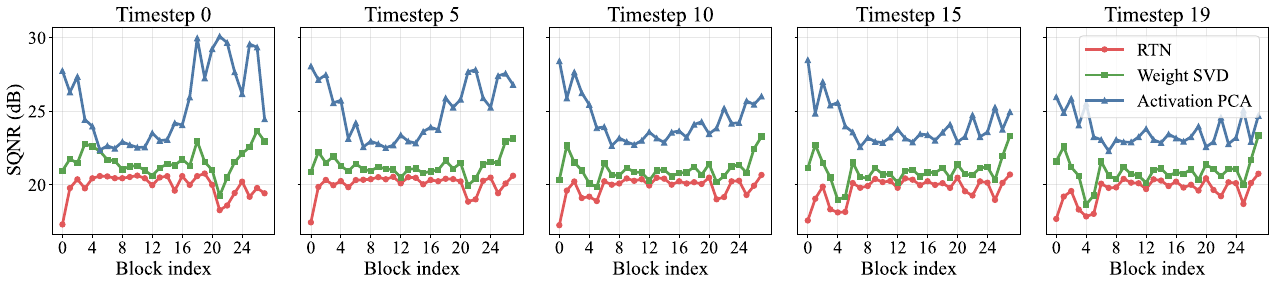}
    
    \vspace*{-2mm}
     \caption{Per-block activation SQNR for the \textit{self-attention output projection} layer across $28$ PixArt-$\Sigma$ transformer blocks at five timesteps. Activation PCA outperforms Weight SVD by $4$\texttt{–}$10$~dB and RTN by $6$\texttt{–}$12$~dB consistently across blocks and timesteps.}
     \label{fig:sqnr_sa_out}
     \vspace*{-4mm}
\end{figure*}

\vspace*{-2mm}
\paragraph{Diffusion models} Diffusion models~\citep{sohl2015deep, ho2020denoising} generate samples via iterative denoising and have achieved strong performance in text-to-image tasks~\citep{balaji2022ediff, rombach2022high, podell2023sdxl}. Early diffusion models primarily relied on UNet-based architectures~\citep{ronneberger2015u}. More recent work has explored transformer-based designs, leading to DiTs~\citep{peebles2023scalable, bao2023all}, which improve scalability and generation quality in models such as PixArt~\citep{chen2023pixart, chen2024pixart}, FLUX~\citep{flux2024}, and Stable Diffusion 3~\citep{esser2024scaling}. Existing efforts have explored accelerating diffusion models through faster samplers~\citep{song2020denoising, lu2022dpm, lu2025dpm}, distillation~\citep{salimans2022progressive, song2024improved, luo2023latent, sauer2024adversarial, yin2024one}, and system-level optimizations, including efficient architectures~\citep{kim2024bk, li2023snapfusion}, sparsification \citep{wang2025sparsedm, yuan2024ditfastattn}, computation optimization via caching and token reduction methods~\citep{ma2024deepcache, bolya2023token}, and quantization~\citep{shang2023post, he2023ptqd, li2025svdquant}. In this work, we focus on quantization, which provides an effective trade-off between compression and generation quality, reducing memory and latency with minimal degradation.

\vspace*{-2mm}
\paragraph{Quantization}  PTQ reduces the precision of weights and activations to improve efficiency without retraining. In LLMs, methods such as GPTQ~\citep{frantar2023gptq} and AWQ~\citep{lin2024awq} focus on weight quantization, while SmoothQuant~\citep{xiao2023smoothquant} and LLM.int8()~\citep{dettmers2022gpt3} target both activation and weight quantization and address activation outliers via scaling or mixed precision. DualQuant~\citep{salmani2026dualquant} quantizes both weights and activations without any calibration data, using per-row and per-column scales learned from weight reconstruction error.
Rotation-based approaches improve activation quantization by suppressing outliers through orthogonal transformations. QuaRot~\citep{ashkboos2024quarot} applies fixed Hadamard rotations to enable end-to-end 4-bit quantization, while TurboQuant~\citep{turboquant2026} leverages random orthogonal rotations for KV cache quantization. ResQ~\citep{saxena2024resq} employs PCA-based rotations to decompose weights and activations into high- and low-variance subspaces, enabling mixed-precision quantization that preserves salient components, and SpinQuant~\citep{liu2025spinquant} learns rotation matrices to reduce the gap to full precision. Other approaches focus on choosing numeric formats that better match the distributions of weights and activations. F8Net~\citep{jin2022f8net} assigns layer-specific fixed-point formats to weights and activations through quantization-aware training, while Sharify et al.~\cite{sharify2024mx} demonstrate that block-scaled microscaling (MX) formats outperform integer formats at the same bit-width for post-training quantization of LLMs. Despite their effectiveness in LLMs, these approaches do not directly extend to diffusion models, where the iterative denoising process introduces pronounced timestep-dependent variation in activation statistics.


Unlike LLMs with relatively stable distributions, diffusion models exhibit large shifts across timesteps, making static quantization ineffective. Early works such as Q-Diffusion~\citep{li2023q-diffusion} and PTQ4DM~\citep{shang2023post} address this via timestep-aware calibration, while PTQD~\citep{he2023ptqd} accounts for error accumulation across denoising steps. In addition to timestep-dependent variation, the adoption of DiTs~\citep{peebles2023scalable} introduces further challenges, including token-wise spatial variance, channel-wise outliers, and condition-wise variation from Classifier-Free Guidance~\citep{ho2022classifier}. Recent methods tackle these issues from different angles: PTQ4DiT~\citep{wu2024ptq4dit} and DiTAS~\citep{dong2025ditas} mitigate outliers via transformation and smoothing, \mbox{ViDiT-Q}~\citep{zhao2025viditq} and DGQ~\citep{ryu2025dgq} introduce fine-grained dynamic quantization, and Q-DiT~\citep{chen2025q} adapts quantization granularity based on activation statistics. 
FP4DiT~\citep{chen2025fp4dit} explores floating-point quantization (W4A6) on DiTs to better align with non-uniform weight/activation distributions and MSFP~\citep{zhao2025msfp} pursues 4-bit FP quantization through timestep-aware LoRA fine-tuning. Rotation-based PTQ methods such as ConvRot~\citep{huang2025convrot} and LRQ-DiT~\citep{yang2025lrq} use Hadamard and log-rotation transforms to suppress activation outliers in DiTs under low-bit settings. More recently, SVDQuant~\citep{li2025svdquant} handles outliers by separating them into a high-precision low-rank branch while quantizing the residual to 4-bit precision, enabling end-to-end W4A4 inference. 

Despite these advances, a gap to full-precision performance remains under aggressive low-bit settings, as existing methods address individual challenges in isolation. For example, SVDQuant~\citep{li2025svdquant} at INT4, applied to PixArt-$\Sigma$~\citep{chen2024pixart}, exhibits ${\sim}14\%$ higher FID score than the $16$-bit floating-point baseline, even after nontrivial optimization. In this work, we propose a PCA-based approach that jointly captures timestep-, token-, and channel-wise variations for accurate and efficient 4-bit quantization.
\section{Quantization preliminaries} 
\label{sec:quantization_preliminary}


Quantization is a model compression technique that reduces the numerical precision of tensor elements, typically from 32-bit or 16-bit floating-point (FP32, FP16, BF16) to lower-bit representations such as 8-bit (INT8, FP8), or even 4-bit (INT4, NVFP4, MXFP4) formats, enabling more efficient computation and storage. Formally, given a full-precision tensor $\mathbf{X}\in \mathbb{R}^n$, $b$-bit quantization maps $\mathbf{X}$ to a lower-precision representation $\hat{\mathbf{X}}$ through a quantization function $Q(\cdot)$: \vspace{-1mm} 
{\small
\begin{equation}
\hat{\mathbf{X}} = Q(\mathbf{X}) = \left\lfloor \frac{\mathbf{X} - z}{s} \right\rceil \cdot s + z,
\label{eq:quant}
\end{equation}}
where $\lfloor \cdot \rceil$ denotes a round-and-clip operation mapping values to $[q_{\min}, q_{\max}]$, $s$ is the scale, and $z$ is the zero-point. For symmetric integer quantization, $z = 0$ and $s = \max(|\mathbf{X}|)/q_{\max}$ with $q_{\max} = 2^{b-1}-1$; for asymmetric quantization, $z = \min(\mathbf{X})$ and $s = (\max(\mathbf{X}) - \min(\mathbf{X}))/(q_{\max} - q_{\min})$ with $q_{\min} = 0$ and $q_{\max} = 2^b - 1$. For floating-point formats, $q_{\max}$ depends on exponent and mantissa bits (e.g., a 4-bit format with 1-bit mantissa and 2-bit exponent yields $q_{\max} = 6$). Following prior work~\citep{li2025svdquant}, we use symmetric quantization for weights and asymmetric for activations. For simplicity, throughout the paper we formulate quantization in the symmetric form (i.e., $z = 0$); the asymmetric case follows analogously with an additional $z$ term. Following prior work~\citep{lin2025qserve, li2025svdquant}, we denote quantization formats as $\mathbf{W}b_w\mathbf{A}b_a$ where $b_w$ and $b_a$ denote the bit widths of weights and activations, respectively, and use \textit{INT} and \textit{FP} to represent integer and floating-point data types.

\section{Method}
\label{sec:method}
\subsection{Problem formulation}
Consider a linear layer $\mathbf{Y} = \mathbf{XW}$ where $\mathbf{X} \in \mathbb{R}^{B \times m}$ is the input activation and $\mathbf{W} \in \mathbb{R}^{m \times n}$ is the weight matrix; $B$ represents the batch size, and $m$ and $n$ denote the input and output channels, respectively. We aim to quantize both weights ($b_w$ bits) and activations ($b_a$ bits) while minimizing the output error. Following ViDiT-Q~\citep{zhao2025viditq}, we approximate this objective by minimizing the layer-wise quantization error of the weights $\mathbf{W}$ and activations $\mathbf{X}$:
{\small{
\begin{equation}
\min_{Q(\mathbf{W}^{(l)}),\, Q(\mathbf{X}^{(l)})} \sum_{l}^{L}
\Big(\; \underbrace{\| \mathbf{W}^{(l)} - Q(\mathbf{W}^{(l)}) \|_F^2}_{\text{weight quantization error}} + \underbrace{\| \mathbf{X}^{(l)} - Q(\mathbf{X}^{(l)}) \|_F^2}_{\text{activation quantization error}} \Big),
\label{eq:layerwise}
\end{equation}}}
where $\mathbf{W}^{(l)}$ and $\mathbf{X}^{(l)}$ represent the weights and activations of layer $l$, respectively.

As shown by prior work~\citep{xiao2023smoothquant, lin2024awq, zhao2025viditq}, weights are relatively easy to quantize and methods such as GPTQ~\citep{frantar2023gptq} handle them effectively. In contrast, activation quantization is more challenging, as discussed in Section~\ref{sec:introduction}, and in this work we focus on mitigating it. Note that Eq.~\ref{eq:layerwise} is a layer-wise surrogate adopted for clarity; in practice, our weight quantization uses GPTQ~\citep{frantar2023gptq}, which minimizes the activation-weighted output error $\|\mathbf{XW} - \mathbf{X\hat{W}}\|_F^2$.

\subsection{DiRotQ: PCA-guided rotated quantization}
To address the challenge of activation quantization, DiRotQ rotates activations $\mathbf{X}$ into the PCA basis $\mathbf{U}$, where $\mathbf{U} \in \mathbb{R}^{m \times m}$ is an orthogonal matrix satisfying $\mathbf{U}^\top \mathbf{U} = \mathbf{I}$. This transformation decorrelates channels and separates high- and low-variance components (see Appendix~\ref{sec:pca_details} for details): 
{
\begin{equation}
\hat{\mathbf{X}} = \mathbf{X} \mathbf{U}, \qquad \mathbf{U} = \left[\mathbf{U}_h \;\; \mathbf{U}_l\right],  \qquad \mathbf{U} \in \mathbb{R}^{m \times m},
\end{equation}}
where $\mathbf{U}_h \in \mathbb{R}^{m \times k}$ contains the top-$k$ eigenvectors corresponding to the largest eigenvalues and captures the high-variance components, and $\mathbf{U}_l \in \mathbb{R}^{m \times (m-k)}$ spans the residual low-variance subspace. The top-$k$ high-variance components, $\hat{\mathbf{X}}_h = \mathbf{X}\mathbf{U}_h$, are retained in high precision (e.g., 16-bit) to preserve the dominant activation information, while the residual components, $\hat{\mathbf{X}}_l = \mathbf{X}\mathbf{U}_l$, are quantized to 4-bit precision. Here, $k$ denotes the rank of the high-precision subspace, controlling the number of components assigned higher precision. In practice, we set $k = rm$, where $r$ is a small ratio (e.g., $r\approx10\%$).

Within the low-precision residual subspace, although the variance is small, aggressive 4-bit quantization can still introduce significant error due to the limited number of quantization levels. To mitigate this, we apply a random orthogonal rotation $\mathbf{R} \in \mathbb{R}^{(m-k)\times(m-k)}$, yielding $\tilde{\mathbf{X}}_l = \hat{\mathbf{X}}_l \mathbf{R}$, which further equalizes variance across channels and improves the efficiency of a shared quantization scale. Note that since $\mathbf{U}_l$ has orthonormal columns and $\mathbf{R}$ is orthogonal, their composition $\mathbf{U}_l \mathbf{R}$ also forms an orthonormal basis for the $(m-k)$-dimensional residual subspace (see Appendix~\ref{app:proof} for proof).


Therefore, the combined rotation is $\mathbf{V} = \left[\mathbf{U}_h \;\; \mathbf{U}_l \mathbf{R}\right]$ and $\mathbf{X}\mathbf{W}$ can be approximated as:
{
\begin{equation}
\begin{split}
\mathbf{Y} = \mathbf{XW} = \mathbf{XVV}^\top \mathbf{W} &= \mathbf{X}\left[\mathbf{U}_h \;\; \mathbf{U}_l \mathbf{R}\right] \left[\mathbf{U}_h \;\; \mathbf{U}_l \mathbf{R}\right]^\top \mathbf{W} \\
&= \left[\mathbf{X}\mathbf{U}_h \;\; \mathbf{X}\mathbf{U}_l \mathbf{R}\right] \left[
\begin{array}{c}
\mathbf{U}_h^\top \\
(\mathbf{U}_l\mathbf{R})^\top
\end{array}
\right] \mathbf{W} \\
&= [\mathbf{X}\mathbf{U}_h \;\; \tilde{\mathbf{X}}_l] \left[
\begin{array}{c}
\mathbf{U}_h^\top \mathbf{W}\\
(\mathbf{U}_l\mathbf{R})^\top \mathbf{W}
\end{array}
\right] \\
  \end{split}
\end{equation}}%
Since the rotation matrices $\mathbf{U}_l$ and $\mathbf{R}$, as well as the weight matrix $\mathbf{W}$, are static, the transpose of the combined rotation 
$(\mathbf{U}_l\mathbf{R})^\top$ can be absorbed into the weights offline: $\tilde{\mathbf{W}_l} = (\mathbf{U}_l\mathbf{R})^\top \mathbf{W}$. Similarly, $\mathbf{U}_h$ can be fused into $\mathbf{W}$ offline: $\tilde{\mathbf{W}}_h = \mathbf{U}_h^\top \mathbf{W}$, yielding the forward computation:
{
\begin{equation}
\begin{split}
\mathbf{Y} &= [\mathbf{X}\mathbf{U}_h \;\; \tilde{\mathbf{X}}_l] \left[\begin{array}{c}
\tilde{\mathbf{W}}_h\\
\tilde{\mathbf{W}_l}
\end{array} \right] \approx [\mathbf{X}\mathbf{U}_h \;\; Q(\tilde{\mathbf{X}}_l)] \left[\begin{array}{c}
\tilde{\mathbf{W}}_h\\
Q(\tilde{\mathbf{W}}_l)
\end{array} \right]
\end{split}
\label{eq:weight_fusion}
\end{equation}}%
The forward rotation $\mathbf{X}\mathbf{V}$ remains \textit{online} because intermediate nonlinearities in DiTs, such as adaptive LayerNorm~\citep{peebles2023scalable}, softmax, or GELU~\citep{hendrycks2016gaussian}, prevent fusing it into preceding weights, while the transpose rotation is fused into the same layer's weights at zero inference cost. 
Note that the PCA basis $\mathbf{U} = [\mathbf{U}_h \;\; \mathbf{U}_l]$ is computed per linear layer from calibration activations aggregated across all denoising timesteps, capturing the principal directions of the \emph{time-averaged activations.}
In contrast, $\mathbf{R}$ is \emph{data-independent}: for each distinct residual dimension $(m-k)$ in the network (e.g., attention hidden size, FFN intermediate size), we sample a random orthogonal matrix once using a fixed seed, and reuse it across all layers of that dimension. Both \emph{$\mathbf{U}$ and $\mathbf{R}$ are fixed across timesteps}. 
Sharing $\mathbf{U}$ across timesteps is required by offline weight fusion (Eq.~\ref{eq:weight_fusion}) as a per-timestep basis would scale weight memory by the number of denoising steps, negating low-bit savings. This design preserves high-variance components with minimal error while efficiently quantizing low-variance residuals, thereby reducing overall activation quantization error and maintaining model accuracy.

\section{Experiments}
\label{sec:experiments}

\begin{figure*}[t!]
    \centering
    \includegraphics[scale=0.59]{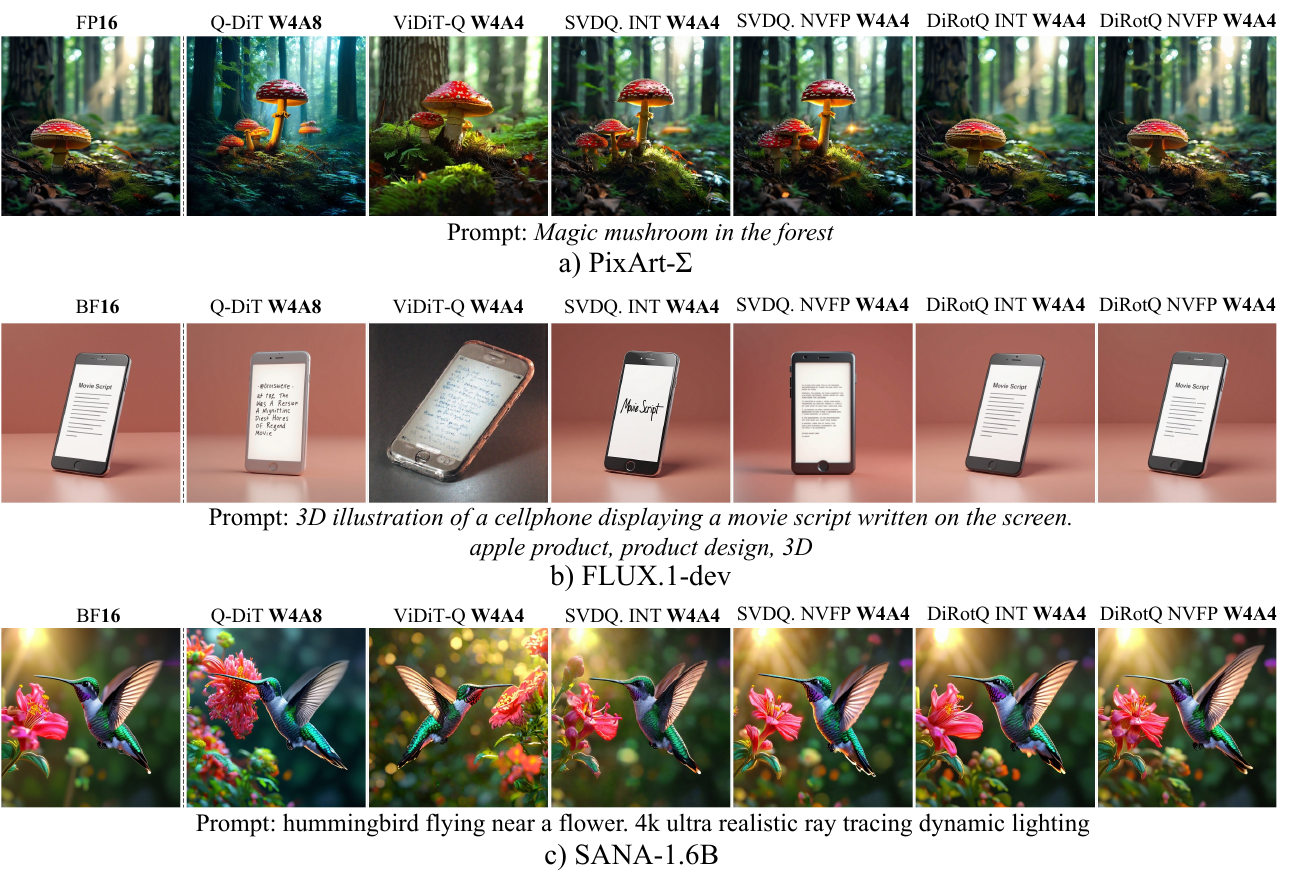}
    
    \vspace*{-3mm} \caption{Visual comparison on MJHQ-30K across three models. DiRotQ consistently outperforms prior techniques, surpasses the W4A8 method, and produces results closest to the $16$-bit baselines.}
    
    \vspace*{-4mm} \label{fig:visual_three_models_mjhq_main}
\end{figure*}

\subsection{Setup}
\label{sec:setup}

\paragraph{Models and Datasets}
We evaluate DiRotQ on three text-to-image DiT~\citep{peebles2023scalable} models: \texttt{PixArt-$\Sigma$}~\citep{chen2024pixart}, \texttt{FLUX.1-dev}~\citep{flux2024}, and \texttt{SANA}~\citep{xie2025sana}. \texttt{PixArt-$\Sigma$} is a 600M-parameter DiT optimized for high-resolution generation, while \texttt{FLUX.1-dev} is a large 12B-parameter model, and \texttt{SANA} is a 1.6B-parameter model designed for improved scalability. We report results in each model’s native precision (FP16 for PixArt-$\Sigma$, BF16 for FLUX.1-dev and SANA).
Following prior work~\citep{li2023q-diffusion, zhao2024mixdq, li2025svdquant, yang2025lrq}, we use 128 prompts from the \texttt{COCO Captions} dataset~\citep{chen2015microsoft} for calibration. We evaluate generalization on two diverse prompt sets: \texttt{MJHQ-30K}~\citep{li2024playground}, which captures diverse artistic styles, and \texttt{sDCI}~\citep{urbanek2024picture}, which focuses on realistic and detailed captions. For both datasets, we randomly sample 5K prompts for evaluation. See Appendix~\ref{app:model-generation-setting-dataset-details} for details on models and datasets.

\vspace*{-2mm}
\paragraph{Baselines}
We compare DiRotQ with the following PTQ techniques:
\begin{itemize}[nosep, wide=0pt, leftmargin=*]
    \item \texttt{PTQ4DiT}~\citep{wu2024ptq4dit} enables W4A8 quantization for DiTs by addressing salient channel outliers and temporal activation variation via channel-wise calibration. However, we do not apply it to FLUX.1-dev due to its 12B parameter scale, where layer-wise AdaRound~\citep{nagel2020up} calibration becomes prohibitively expensive (over 60 hours per run on an A100 GPU), making iterative tuning impractical.
    \item \texttt{Q-DiT}~\citep{chen2025q} addresses spatial and temporal variance in DiTs by combining automatic quantization granularity allocation across channels with sample-wise dynamic activation quantization across timesteps, improving robustness and accuracy under low-bit settings (e.g., W4A8).
    \item \texttt{ViDiT-Q}~\citep{zhao2025viditq} addresses activation outliers using per-token quantization and smoothing~\citep{xiao2023smoothquant}, combined with rotation-based transformations~\citep{ashkboos2024quarot, liu2025spinquant}, achieving lossless W8A8 quantization on PixArt-$\Sigma$ and strong performance at lower bit-widths (e.g., W4A8).
    \item \texttt{SVDQuant}~\citep{li2025svdquant} first shifts activation outliers to weights via smoothing~\citep{xiao2023smoothquant, lin2024awq}, then uses \mbox{SVD-based}~\citep{eckart1936approximation} low-rank decomposition to split tensors into a high-precision branch for outliers and a low-bit branch for residuals, enabling efficient W4A4 quantization of DiT models.
\end{itemize}

\vspace*{-2mm}
\paragraph{Metrics}
Following prior work~\citep{li2023q-diffusion, li2025svdquant, yang2025lrq}, we evaluate model performance using \textit{quality} and \textit{similarity} metrics. For quality, we report \mbox{\texttt{Fréchet Inception Distance}} (\texttt{FID}~\citep{heusel2017gans}, $\downarrow$), which measures distributional distance between generated and real images, and \texttt{Image Reward} \mbox{(\texttt{IR}~\citep{xu2023imagereward}, $\uparrow$)}, which reflects alignment with human preferences. For similarity, we use \texttt{Learned Perceptual Image Patch Similarity} \mbox{(\texttt{LPIPS}~\citep{zhang2018unreasonable}, $\downarrow$)} and \texttt{Peak Signal-to-Noise Ratio} (\texttt{PSNR}, $\uparrow$), measuring perceptual and pixel-level fidelity, respectively. Since LPIPS and PSNR are computed against 16-bit reference outputs, they are undefined for reference rows (marked as -- in tables). 

For holistic evaluation, we further employ a \mbox{\texttt{VLM-as-a-Judge protocol}} inspired by VIEScore~\citep{ku2024viescore}, using a Multimodal Large Language Model (MLLM)~\citep{yin2024survey} to assess Semantic Consistency (SC) and \mbox{Perceptual Quality (PQ)} on a \mbox{$1$\texttt{-}$10$ scale ($\uparrow$)}. More details on additional metrics, including \mbox{CLIP IQA~\citep{wang2023exploring}}, CLIP Score~\citep{hessel2021clipscore}, and SSIM~\citep{wang2004image}, are provided in Appendix~\ref{app:additiona-eval-metrics}.

\begin{table*}[t!]
\centering
\caption{Comparison of FID, Image Reward (IR), LPIPS, and PSNR on the MJHQ-30K and sDCI datasets across different models. All results are obtained using the official codebases of each method. All methods, except PTQ4DiT and ViDiT-Q, employ GPTQ~\citep{frantar2023gptq} for weight quantization. $\uparrow$: higher is better, and $\downarrow$: lower is better. DiRotQ closely matches or outperforms prior PTQ methods across different models and datasets, achieving the best quality–similarity trade-off at 4-bit precision.}
\label{tab:visual-quality-results-main}
\resizebox{0.99\textwidth}{!}{
\begin{tabular}{clccccccccc}
\toprule
 \multirow{3}{*}[-1.0ex]{\textbf{Model}} & \multirow{3}{*}[-1.0ex]{\textbf{Precision}} & \multirow{3}{*}[-1.0ex]{\textbf{Method}} & \multicolumn{4}{c}{\textbf{MJHQ-30K}} & \multicolumn{4}{c}{\textbf{sDCI}}  \\ \cmidrule(r){4-7}  \cmidrule(l){8-11} 
 & & & \multicolumn{2}{c}{\textbf{Quality}} &  \multicolumn{2}{c}{\textbf{Similarity}} & \multicolumn{2}{c}{\textbf{Quality}} &  \multicolumn{2}{c}{\textbf{Similarity}} \\ \cmidrule(r){4-5}  \cmidrule(lr){6-7}  \cmidrule(lr){8-9}  \cmidrule(l){10-11} 
  & & & \textbf{FID ($\downarrow$)} & \textbf{IR ($\uparrow$)} & \textbf{LPIPS ($\downarrow$)} & \textbf{PSNR ($\uparrow$)} & \textbf{FID ($\downarrow$)} & \textbf{IR ($\uparrow$)} & \textbf{LPIPS ($\downarrow$)} & \textbf{PSNR ($\uparrow$)} \\ \midrule \midrule
  & FP16     & - & 16.6 & 0.952 & - & - & 24.8 & 0.963 & - & - \\ \cmidrule{2-11}
  & PTQ4DiT & INT W4A8 & 39.8 & 0.572 & 0.531 & 12.8 & 47.1 & 0.560 & 0.649 & 11.0 \\
  & Q-DiT    & INT W4A8 & 19.6 & 0.874 & 0.500 & 13.8 & 26.1 & 0.945 & 0.518 & 13.0 \\
  & \cellcolor[HTML]{e3f4f7}DiRotQ & \cellcolor[HTML]{e3f4f7}INT W4A8 & \cellcolor[HTML]{e3f4f7}\textbf{16.5} & \cellcolor[HTML]{e3f4f7}\textbf{0.942} & \cellcolor[HTML]{e3f4f7}\textbf{0.222} & \cellcolor[HTML]{e3f4f7}\textbf{19.5} & \cellcolor[HTML]{e3f4f7}\textbf{24.5} & \cellcolor[HTML]{e3f4f7}\textbf{0.956} & \cellcolor[HTML]{e3f4f7}\textbf{0.236} & \cellcolor[HTML]{e3f4f7}\textbf{18.4} \\ \cmidrule{2-11}
  & ViDiT-Q  & INT W4A4 & 21.1 & 0.896 & 0.698 & 10.0 & 26.5 & 0.942 & 0.717 & 9.3 \\
  & SVDQuant & INT W4A4 & 18.9 & 0.871 & 0.319 & 17.6 & 25.8 & 0.922 & 0.352 & 16.4 \\
   & \cellcolor[HTML]{e3f4f7}DiRotQ & \cellcolor[HTML]{e3f4f7}INT W4A4 & \cellcolor[HTML]{e3f4f7}\textbf{15.9} & \cellcolor[HTML]{e3f4f7}\textbf{0.941} & \cellcolor[HTML]{e3f4f7}\textbf{0.236} & \cellcolor[HTML]{e3f4f7}\textbf{19.1} & \cellcolor[HTML]{e3f4f7}\textbf{23.1} & \cellcolor[HTML]{e3f4f7}\textbf{0.965} & \cellcolor[HTML]{e3f4f7}\textbf{0.251} & \cellcolor[HTML]{e3f4f7}\textbf{18.0} \\ \cmidrule{2-11}
  \multirow{-9}{*}{PixArt-$\Sigma$} & SVDQuant & NVFP W4A4 & 16.8 & 0.933 & 0.264 & 18.6 & 23.2 & \textbf{0.964} & 0.288 & 17.4 \\
  \multirow{-9}{*}{(20 Steps)}& \cellcolor[HTML]{e3f4f7}DiRotQ   & \cellcolor[HTML]{e3f4f7}NVFP W4A4 &  \cellcolor[HTML]{e3f4f7}\textbf{16.0} & \cellcolor[HTML]{e3f4f7}\textbf{0.946} & \cellcolor[HTML]{e3f4f7}\textbf{0.224} & \cellcolor[HTML]{e3f4f7}\textbf{19.6} & \cellcolor[HTML]{e3f4f7}\textbf{23.0} & \cellcolor[HTML]{e3f4f7}\textbf{0.964} & \cellcolor[HTML]{e3f4f7}\textbf{0.241} & \cellcolor[HTML]{e3f4f7}\textbf{18.4}  \\ \midrule \midrule

  & BF16     & - & 19.9 & 0.939 & - & - & 25.0 & 1.02 & - & - \\ \cmidrule{2-11}
  & Q-DiT    & INT W4A8 & 19.9 & 0.909 & 0.213 & 21.8 & 25.0 & 1.00 & 0.240 & 20.1 \\
  & \cellcolor[HTML]{e3f4f7}DiRotQ   & \cellcolor[HTML]{e3f4f7}INT W4A8 & \cellcolor[HTML]{e3f4f7}\textbf{19.8} & \cellcolor[HTML]{e3f4f7}\textbf{0.933} & \cellcolor[HTML]{e3f4f7}\textbf{0.132} & \cellcolor[HTML]{e3f4f7}\textbf{24.9} & \cellcolor[HTML]{e3f4f7}\textbf{24.8} & \cellcolor[HTML]{e3f4f7}\textbf{1.02} & \cellcolor[HTML]{e3f4f7}\textbf{0.146} & \cellcolor[HTML]{e3f4f7}\textbf{23.2} \\ \cmidrule{2-11}
  & ViDiT-Q  & INT W4A4 & 36.3 & 0.382 & 0.672 & 13.6 & 32.8 & 0.456 & 0.671 & 12.6 \\
  & SVDQuant & INT W4A4 & \textbf{19.7} & 0.920 & 0.241 & 21.0 & \textbf{24.7} & \textbf{1.01} & 0.273 & 19.2 \\
  & \cellcolor[HTML]{e3f4f7}DiRotQ & \cellcolor[HTML]{e3f4f7}INT W4A4 & \cellcolor[HTML]{e3f4f7}19.8 & \cellcolor[HTML]{e3f4f7}\textbf{0.934} & \cellcolor[HTML]{e3f4f7}\textbf{0.139} & \cellcolor[HTML]{e3f4f7}\textbf{24.6} & \cellcolor[HTML]{e3f4f7}24.9 & \cellcolor[HTML]{e3f4f7}\textbf{1.01} & \cellcolor[HTML]{e3f4f7}\textbf{0.156} & \cellcolor[HTML]{e3f4f7}\textbf{22.8} \\ \cmidrule{2-11}
  \multirow{-9}{*}{FLUX.1-dev} & SVDQuant & NVFP W4A4 & \textbf{19.9} & 0.930 & 0.180 & 22.9 & \textbf{24.8} & \textbf{1.01} & 0.201 & 21.2 \\
  \multirow{-9}{*}{(25 Steps)} & \cellcolor[HTML]{e3f4f7}DiRotQ  & \cellcolor[HTML]{e3f4f7}NVFP W4A4 & \cellcolor[HTML]{e3f4f7}\textbf{19.9} & \cellcolor[HTML]{e3f4f7}\textbf{0.932} & \cellcolor[HTML]{e3f4f7}\textbf{0.144} & \cellcolor[HTML]{e3f4f7}\textbf{24.4} & \cellcolor[HTML]{e3f4f7}25.0 & \cellcolor[HTML]{e3f4f7}\textbf{1.01} & \cellcolor[HTML]{e3f4f7}\textbf{0.159} & \cellcolor[HTML]{e3f4f7}\textbf{22.7} \\ \midrule \midrule

  & BF16     & - & 16.0 & 1.10 & - & - & 22.8 & 1.07 & - & - \\ \cmidrule{2-11}
  & PTQ4DiT  & INT W4A8 & 24.8 & 0.647 & 0.668 & 10.3 & 27.6 & 0.716 & 0.686 & 9.3 \\
  & Q-DiT    & INT W4A8 & 16.5 & \textbf{1.11} & 0.247 & 17.5 & \textbf{22.6} & \textbf{1.06} & 0.277 & 15.8  \\
  & \cellcolor[HTML]{e3f4f7}DiRotQ   & \cellcolor[HTML]{e3f4f7}INT W4A8 & \cellcolor[HTML]{e3f4f7}\textbf{16.0} & \cellcolor[HTML]{e3f4f7}1.10 & \cellcolor[HTML]{e3f4f7}\textbf{0.078} & \cellcolor[HTML]{e3f4f7}\textbf{24.1} & \cellcolor[HTML]{e3f4f7}22.9 & \cellcolor[HTML]{e3f4f7}\textbf{1.06} & \cellcolor[HTML]{e3f4f7}\textbf{0.097} & \cellcolor[HTML]{e3f4f7}\textbf{21.9} \\ \cmidrule{2-11}
  & ViDiT-Q  & INT W4A4 & 16.1 & 1.07 & 0.242 & 17.7 & 22.7 & 1.03 & 0.273 & 15.9 \\
  & SVDQuant & INT W4A4 & 16.1 & 1.08 & 0.222 & 18.1 & \textbf{21.5} & \textbf{1.06} & 0.235 & 16.8 \\
  & \cellcolor[HTML]{e3f4f7}DiRotQ & \cellcolor[HTML]{e3f4f7}INT W4A4 & \cellcolor[HTML]{e3f4f7}\textbf{16.0} & \cellcolor[HTML]{e3f4f7}\textbf{1.10} & \cellcolor[HTML]{e3f4f7}\textbf{0.086} & \cellcolor[HTML]{e3f4f7}\textbf{23.5} & \cellcolor[HTML]{e3f4f7}22.9 & \cellcolor[HTML]{e3f4f7}\textbf{1.06} & \cellcolor[HTML]{e3f4f7}\textbf{0.104} & \cellcolor[HTML]{e3f4f7}\textbf{21.4} \\ \cmidrule{2-11}
  \multirow{-9}{*}{SANA-1.6B} & SVDQuant & NVFP W4A4 & \textbf{15.9} & \textbf{1.09} & 0.162 & 19.8 & 22.8 & \textbf{1.06} & 0.181 & 18.1 \\
  \multirow{-9}{*}{(20 Steps)} & \cellcolor[HTML]{e3f4f7}DiRotQ   & \cellcolor[HTML]{e3f4f7}NVFP W4A4 & \cellcolor[HTML]{e3f4f7}16.0 & \cellcolor[HTML]{e3f4f7}\textbf{1.09} & \cellcolor[HTML]{e3f4f7}\textbf{0.092} & \cellcolor[HTML]{e3f4f7}\textbf{23.2} & \cellcolor[HTML]{e3f4f7}\textbf{22.4} & \cellcolor[HTML]{e3f4f7}\textbf{1.06} & \cellcolor[HTML]{e3f4f7}\textbf{0.108} & \cellcolor[HTML]{e3f4f7}\textbf{21.2} \\ \bottomrule
  \end{tabular}
} \vspace*{-2mm}
\end{table*}

\vspace*{-2mm}
\paragraph{Implementation details}
Following SVDQuant~\citep{li2025svdquant}, we adopt \textit{per-token dynamic} activation quantization and \textit{per-channel static} weight quantization for $8$-bit settings. For $4$-bit INT, we apply \textit{per-group} quantization to both activations and weights (group size $64$) with $16$-bit scaling factors, using asymmetric activation quantization. For $4$-bit FP, we use NVFP4~\citep{nvidia_blackwell_2024}, which supports group size $16$ with FP8 (E4M3) scaling. For VLM-as-a-Judge evaluation, the original VIEScore backbone, GPT-4o~\citep{openai2024gpt4o}, is proprietary, and running it at our evaluation scale is costly, as our experiments involve scoring thousands of images across multiple methods and quantization formats. We therefore adopt two open-source judges, \texttt{Qwen2-VL-7B-Instruct}~\citep{wang2024qwen2} and \texttt{InternVL2-8B}~\citep{chen2024expanding}, as our primary, fully reproducible evaluators. To assess the robustness of our conclusions to the choice of judge, we additionally evaluate with \texttt{Gemini~3.6~Flash}~\citep{google2026gemini36flash}, a recent closed-source judge from a different model family, whose low inference cost makes it practical for large-scale evaluation. We approximate the VIEScore protocol by: (i) prompting each model to generate a step-by-step rationale \textit{prior} to assigning numerical scores~\citep{zheng2023judging}, (ii) using low-temperature sampling to reduce score variance, and (iii) computing the overall score as the geometric mean of SC and PQ. All three judges are evaluated under the same protocol. Appendix~\ref{app:additional_implementation_details} provides full details of the quantization scheme and VLM-as-a-Judge evaluation protocol, including scoring scheme and judge prompts.


\vspace*{-2mm}
\subsection{Results}
\label{sec:results}
\paragraph{Visual quality results}
We present quantitative results in Table~\ref{tab:visual-quality-results-main} across multiple models and precision settings, along with visual qualitative comparisons in Figure~\ref{fig:visual_three_models_mjhq_main}. We exclude PTQ4DiT from the qualitative comparisons in the interest of space, as it consistently underperforms relative to the other methods. Figure~\ref{fig:visual_three_models_mjhq_main} illustrates visual comparisons across three models on the MJHQ-30K dataset. We observe that DiRotQ at $4$-bit precision consistently outperforms other methods, even surpassing the W4A8 technique, achieving closer similarity to the $16$-bit models.

Table~\ref{tab:visual-quality-results-main} indicates that DiRotQ closely matches or outperforms prior PTQ methods across all models and datasets, with the most significant gains at W4A4 precision. On PixArt-$\Sigma$, DiRotQ INT W4A4 achieves FID $15.9$ and IR $0.941$ on MJHQ-30K, substantially improving over SVDQuant (FID $18.9$, IR $0.871$), while also delivering higher fidelity (PSNR $19.1$ vs. $17.6$) and lower perceptual error (LPIPS $0.236$ vs. $0.319$). These improvements extend to sDCI and NVFP4 formats, where DiRotQ consistently reduces LPIPS and improves PSNR, indicating better preservation of visual details. For larger models such as FLUX.1-dev and SANA-1.6B, DiRotQ maintains comparable or better FID while significantly improving similarity metrics, highlighting stronger reconstruction quality under aggressive quantization. Notably, DiRotQ at W4A4 substantially closes the gap to BF16 (e.g., FLUX.1-dev on MJHQ-30K, IR 0.934 vs. 0.939), demonstrating that PCA-based rotation effectively mitigates activation quantization error and enables high-quality 4-bit inference across diverse DiT architectures. Additional quantitative metrics, including CLIP Score, CLIP IQA, and SSIM, and visual qualitative comparisons are provided in Appendix~\ref{app:visual_quality_results_appendix} (Table~\ref{tab:visual-quality-results-appendix}, Figures~\ref{fig:pixart_visual},~\ref{fig:flux_visual}, and~\ref{fig:sana_visual}).

\vspace*{-2mm}
\paragraph{Memory saving and speedup}
\label{sec:memory_speedup}

Figure~\ref{fig:memory_speedup} and Figure~\ref{fig:NVFP4_memory_savings_and_speedup} report the memory savings and inference speedups of DiRotQ on FLUX.1-dev. Following SVDQuant~\citep{li2025svdquant}, we evaluate the INT4 configuration on an NVIDIA RTX 4090 desktop GPU~\citep{nvidia_rtx4090} and the NVFP4 configuration on an NVIDIA RTX PRO 6000 Blackwell GPU~\citep{nvidia_rtx6000pro}. For efficient $4$-bit INT matrix multiplication, we use the optimized kernel provided by the \texttt{Nunchaku} inference engine~\citep{li2025svdquant}. For the rotation operations, we implement an efficient Triton~\citep{tillet2019triton} kernel. 
Compared to the BF16 baseline, DiRotQ achieves an overall \textit{$2.1\times$} reduction in DiT inference memory, after accounting for overhead from high-precision branches, rotation matrices, and unquantized normalization layers. The memory savings mainly come from quantizing the weights of the linear layers, as activation memory accounts for only about $2\%$ of the weight memory. As a result, DiRotQ has nearly the same memory footprint as the INT4 weight-only quantization variant (Figure~\ref{fig:memory_speedup}a).
For latency, when the batch size is $1$ and the BF16 model fits on GPU, DiRotQ achieves a \textit{$2.3\times$} end-to-end speedup per denoising timestep compared to BF16. In contrast, the INT4 weight-only baseline provides little to no speedup, since it does not enable low-precision compute and introduces additional overhead from upcasting weights to BF16 (Figure~\ref{fig:memory_speedup}b). When increasing the batch size to $2$, the BF16 model no longer fits entirely in GPU memory and requires per-layer CPU offloading, while DiRotQ still fits entirely in GPU memory. This avoids offloading overhead and results in a \textit{$9.2\times$} speedup over BF16 (Figure~\ref{fig:memory_speedup}d). 
The detailed INT4 latency breakdown is provided in Appendix~\ref{app:inference_latency_breakdown}. On the RTX PRO 6000 Blackwell~\citep{nvidia_rtx6000pro}, the NVFP4 configuration delivers similar $2\times$ memory savings and end-to-end speedups over BF16 at both batch sizes $1$ and $2$; detailed results are provided in Appendix~\ref{app:NVFP4_memory_savings_and_speedup}. Note that the batch-size-$2$ result on the RTX 4090 desktop GPU is intended to demonstrate a memory-constrained deployment scenario rather than a typical single-image use case. This regime is relevant to throughput-oriented generation, CFG-based sampling with joint conditional and unconditional passes~\citep{ho2022classifier}, multi-image workloads such as reference- or condition-guided~\citep{ye2023ip, zhang2023controlnet} or multi-view generation~\citep{shi2024mvdream}, and high-resolution (e.g., $4$K) image generation~\citep{chen2024pixart}, where the memory savings from 4-bit quantization can translate directly into end-to-end speedups. 

\begin{figure*}[t!]
    \centering
    \includegraphics[scale=0.34]{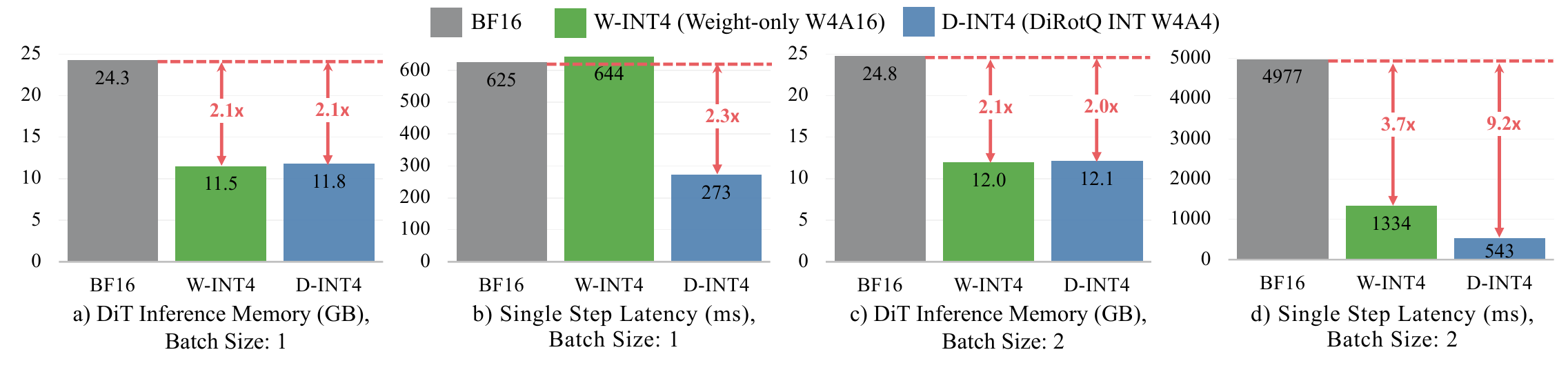}
    \vspace{-2mm}
     \caption{Memory usage and one-step latency for FLUX.1-dev across batch sizes 1 and 2. DiRotQ (INT W4A4) reduces memory by $2.1\times$ and achieves $2.3\times$ speedup at batch size 1 and $9.2\times$ at batch size 2 over BF16; the larger gain at batch size 2 is driven by eliminating CPU offloading.}
     \label{fig:memory_speedup}
     
     \vspace{-3mm}
\end{figure*}

\vspace*{-2mm}
\paragraph{VLM-as-a-Judge}
Figure~\ref{fig:vlm_judge} shows pairwise VLM-as-a-Judge win/tie/loss rates between DiRotQ and SVDQuant on PixArt-$\Sigma$ generations from the MJHQ-30K prompt set, evaluated by the two open-source judges, InternVL2-8B and Qwen2-VL-7B-Instruct, across Semantic Consistency (SC), Perceptual Quality (PQ), and Overall score. Differences below $0.01$ are treated as ties. Under INT4 quantization (Figure~\ref{fig:vlm_judge_int4}), DiRotQ is consistently preferred across all metrics by both judges. The advantage is most pronounced for the Overall score, where InternVL2-8B favors DiRotQ on $48.1\%$ of prompts, compared to $39.9\%$ for SVDQuant. Qwen2-VL shows the same trend ($23.3\%$ vs. $16.2\%$), albeit with a higher tie rate. Under NVFP4 (Figure~\ref{fig:vlm_judge_nvfp4}), the gap narrows as both methods approach full-precision quality and become less distinguishable. Nevertheless, DiRotQ maintains an advantage across most settings, outperforming SVDQuant in 5 out of 6 judge--metric combinations. Together with the INT4 results, these findings suggest that DiRotQ provides more reliable perceptual quality and alignment under low-bit quantization.

We further examine whether the advantage of DiRotQ over SVDQuant is robust to the choice of judge. Repeating the INT4 PixArt-$\Sigma$ evaluation with the closed-source Gemini~3.6~Flash under the same protocol on a 1K-prompt subset, DiRotQ is again preferred over SVDQuant on the Overall score, with a win rate of $58.2\%$ versus $29.9\%$ for SVDQuant ($11.9\%$ ties), yielding a wider margin than either open-source judge (Figure~\ref{fig:gemini_judge_per_category}). The consistent preference across two open-source judges and an independent closed-source model supports the robustness of our findings to the choice of evaluator. Per-axis and per-category breakdowns for Gemini~3.6~Flash, as well as per-category breakdowns for the open-source judges, are reported in Appendix~\ref{app:vlm_as_judge_per_category} (Figures~\ref{fig:vlm_judge_per_category} and~\ref{fig:gemini_judge_per_category}).

\begin{figure}[t!]
    \centering
    \begin{subfigure}{\textwidth}
        \centering
        \includegraphics[scale=0.508]{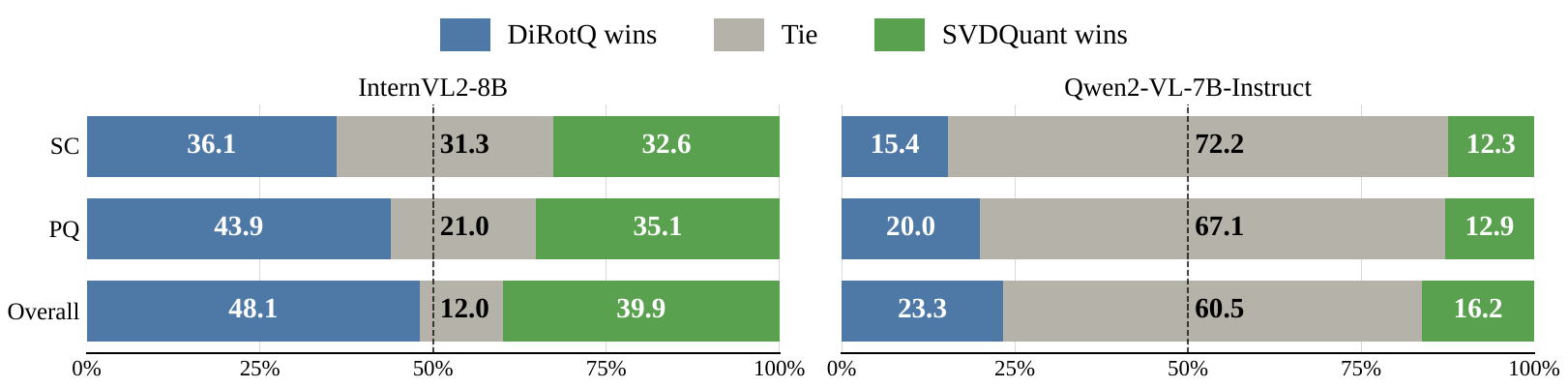}
        
        \vspace*{-1mm} \caption{INT4 quantization.}
        \label{fig:vlm_judge_int4}
    \end{subfigure}
    
    \begin{subfigure}{\textwidth}
        \centering
        \includegraphics[scale=0.508]{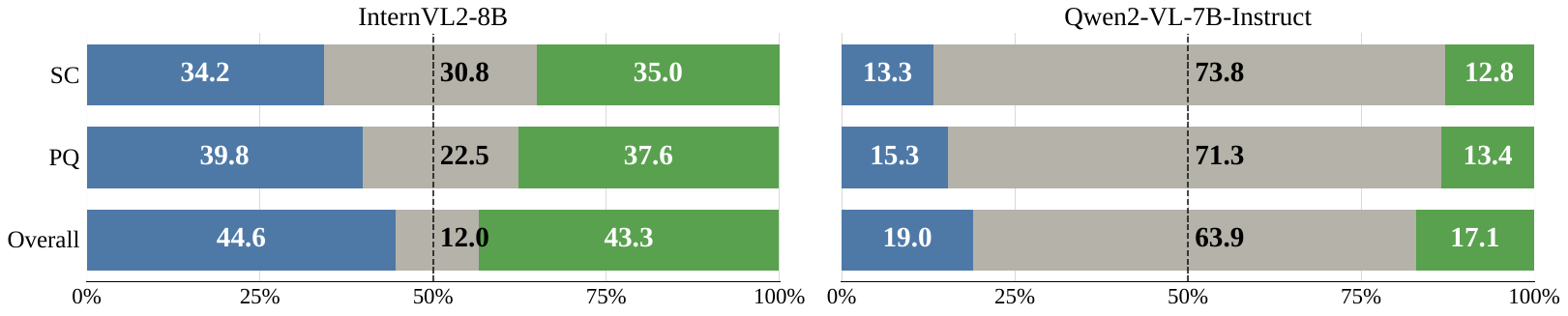}
        
        \vspace*{-1mm} \caption{NVFP4 quantization.}
        \label{fig:vlm_judge_nvfp4}
    \end{subfigure}

     \vspace*{-2mm}
    \caption{VLM-as-a-Judge pairwise comparison between SVDQuant and DiRotQ on PixArt-$\Sigma$ generations from MJHQ-30K prompt set, judged by InternVL2-8B (left) and Qwen2-VL-7B-Instruct (right) along Semantic Consistency (SC), Perceptual Quality (PQ), and the geometric-mean Overall score. (a) INT4 quantization. (b) NVFP4 quantization. Dashed line indicates the $50\%$ parity point; differences below $0.01$ are treated as ties.}
    \vspace*{-4mm}
    \label{fig:vlm_judge}
\end{figure}

\begin{table}[t]
\centering
\caption{Quantitative comparison of DiRotQ settings on PixArt-$\Sigma$ over MJHQ-30K. PCA-based rotation isolates outliers into a high-precision subspace, significantly improving quality over vanilla RTN $4$-bit quantization. The residual orthogonal rotation further equalizes variance in the low-precision subspace, and GPTQ on the weight residual brings DiRotQ closer to FP16 quality.}

\label{tab:ablation}
\resizebox{0.76\textwidth}{!}{
\begin{tabular}{lccc|cccc}
\toprule
\textbf{Precision} & \textbf{PCA} & \textbf{Rotation} & \textbf{GPTQ} & \textbf{FID} ($\downarrow$) & \textbf{IR} ($\uparrow$) & \textbf{LPIPS} ($\downarrow$) & \textbf{PSNR} ($\uparrow$) \\
\midrule
FP16 & -- & -- & -- & 16.6 & 0.952 & -- & -- \\
\midrule
\multirow{4}{*}{INT W4A4}
 & \xmark & \xmark & \xmark & 40.3 & 0.512 & 0.603 & 13.0 \\
 & \cmark & \xmark & \xmark & 19.7 & 0.892 & 0.452 & 14.6 \\
 & \cmark & \cmark & \xmark & 19.2 & 0.904 & 0.438 & 14.8 \\
 & \cellcolor[HTML]{e3f4f7}\cmark & \cellcolor[HTML]{e3f4f7}\cmark & \cellcolor[HTML]{e3f4f7}\cmark & \cellcolor[HTML]{e3f4f7}\textbf{15.9} & \cellcolor[HTML]{e3f4f7}\textbf{0.941} & \cellcolor[HTML]{e3f4f7}\textbf{0.236} & \cellcolor[HTML]{e3f4f7}\textbf{19.1} \\
\midrule
\multirow{4}{*}{NVFP W4A4}
 & \xmark & \xmark & \xmark & 27.7 & 0.710 & 0.515 & 14.7 \\
 & \cmark & \xmark & \xmark & 16.7 & 0.899 & 0.335 & 17.1 \\
 & \cmark & \cmark & \xmark & 16.7 & 0.912 & 0.324 & 17.3 \\
 & \cellcolor[HTML]{e3f4f7}\cmark & \cellcolor[HTML]{e3f4f7}\cmark & \cellcolor[HTML]{e3f4f7}\cmark & \cellcolor[HTML]{e3f4f7}\textbf{16.0} & \cellcolor[HTML]{e3f4f7}\textbf{0.946} & \cellcolor[HTML]{e3f4f7}\textbf{0.224} & \cellcolor[HTML]{e3f4f7}\textbf{19.6} \\
\bottomrule
\end{tabular}
\vspace*{-4mm}
} 
\vspace*{-4mm}
\end{table}

\vspace{-1mm}
\subsection{Ablation study}
\vspace{-1mm}
\paragraph{Component ablation} Table~\ref{tab:ablation} shows the contribution of each DiRotQ component on PixArt-$\Sigma$ over MJHQ-30K. Naive INT W4A4 RTN quantization degrades FID severely from $16.6$ (FP16) to $40.3$, confirming that $4$-bit quantization is highly error-prone without outlier mitigation. Adding PCA-based rotation reduces FID to $19.7$ by isolating outlier energy into a high-precision subspace, while the residual orthogonal rotation provides further consistent gains across all metrics. GPTQ on the weight residual has a synergistic effect with the rotations, bringing DiRotQ to FID $15.9$, on par with FP16, and PSNR $19.1$.
The same trend holds for NVFP4, where the full configuration reaches FID $16.0$ and PSNR $19.6$. 


\vspace*{-1.5mm}
\paragraph{High-precision tail fraction} For PixArt-$\Sigma$ on MJHQ-30K, we sweep the FP16 tail fraction $r$ from $5\%$ to $25\%$ across different layers and denoising timesteps. Activation SQNR improves sharply up to $r{=}10\%$ and saturates beyond, with gains under $0.5$~dB per additional $5\%$ of FP16 allocation. This trend is consistent with the variance concentration of the PCA basis: the leading components capture a large fraction of the activation variance and dominant outliers. Retaining these components in FP16 therefore eliminates most of the quantization error, while the remaining low-variance directions can be quantized to 4-bit with limited impact on SQNR. Increasing $r$ beyond $10\%$ thus provides diminishing quality gains while increasing the size of the high-precision branch and the associated rotation overhead, resulting in higher memory usage and latency. We therefore use $r{=}10\%$ throughout the paper, as it recovers most of the achievable quality while maintaining low overhead (Section~\ref{sec:memory_speedup}). Full per-layer and per-timestep results are provided in Appendix~\ref{app:hp_sweep}.

\vspace*{-1.5mm}
\paragraph{Sensitivity to calibration data and timestep coverage} We evaluate the sensitivity of DiRotQ to i)~\textit{calibration set size}, ii)~\textit{calibration prompt source}, and iii)~\textit{timestep coverage} using per-layer activation SQNR on PixArt-$\Sigma$ with the fixed $10\%$ FP16 budget. Overall, DiRotQ is robust to calibration set size and prompt source, while timestep coverage is the dominant source of sensitivity. Increasing the calibration set size from $16$ to $512$ samples changes the mean activation SQNR by only $0.09$~dB, and using COCO versus MJHQ prompts changes it by only $0.04$~dB. In contrast, SQNR varies by $\sim\!5.8$~dB across the denoising trajectory, reflecting changes in activation statistics across timesteps. These results motivate collecting calibration activations across the full timestep range while using a single shared rotation, providing robustness to practical calibration choices without the memory and deployment overhead of per-timestep rotations. Detailed ablation results and the per-timestep SQNR breakdown are provided in Appendix~\ref{app:ablation-calib}.

\vspace*{-1mm}

\section{Limitations and broader impacts}
\label{sec:Limitations_and_broader_impacts}
DiRotQ takes a step toward enabling efficient deployment of diffusion transformers by significantly reducing memory and compute requirements through 4-bit post-training quantization. This can improve accessibility and allow broader use of large-scale generative models in resource-constrained settings. While our work focuses on text-to-image generation, extending DiRotQ to other modalities such as video generation remains an important direction; however, temporal dependencies across frames introduce additional challenges that are not addressed in this work. More broadly, although our method facilitates wider adoption of generative models, it does not mitigate potential risks of misuse, such as generating misleading or harmful content. Addressing these concerns will require complementary efforts, including safeguards, responsible deployment practices, and further research into robust and trustworthy generative systems.
Additionally, our PCA basis is shared across denoising timesteps, a design imposed by offline weight fusion. While the high-precision tail and per-token activation scales effectively absorb timestep-dependent magnitude variation in our experiments, exploring per-timestep-group rotations or adaptive tail sizes, at the cost of additional weight or computation overhead, remains an interesting direction for future work. 

\vspace*{-1mm}

\section{Conclusion}
\label{sec:conclusion}
In this work, we introduce DiRotQ, a rotation-aware 4-bit PTQ framework specifically designed for diffusion transformers. By leveraging PCA-based transformations to better align activation statistics, DiRotQ effectively mitigates quantization errors arising from timestep-dependent variation and activation outliers. To provide a comprehensive assessment of generative performance, we further integrate a VLM-as-a-Judge evaluation protocol, offering a holistic view of perceptual quality and prompt alignment. Our implementation includes an efficient Triton-based rotation kernel within the inference pipeline, significantly reducing memory footprint while improving runtime performance. Extensive experiments demonstrate that DiRotQ consistently preserves high generation quality across diverse benchmarks. Notably, under the stringent W4A4 setting, DiRotQ achieves a PSNR of 24.6 on the MJHQ-30K dataset using the FLUX.1-dev architecture. Furthermore, we demonstrate the practical scalability of our approach by reducing the memory usage of the 12B FLUX.1-dev model by $2.1\times$ and achieving a $2.3\times$ speedup over the BF16 baseline on a consumer-grade RTX 4090 GPU when quantized to INT4, and delivering comparable $2.0\times$ memory and speedup gains on an RTX PRO 6000 Blackwell GPU when quantized to NVFP4. These results highlight DiRotQ’s ability to narrow the gap to full-precision performance while enabling efficient deployment of large-scale diffusion models on resource-constrained hardware.

\begin{ack}
This work was fully funded by d-Matrix, which also provided all computational resources used in this study, including GPUs for calibration, quantization, image generation, and latency measurements. No external funding, grants, or in-kind contributions were received. The authors declare no competing interests other than the funding and computational support provided by d-Matrix.
\end{ack}



\bibliographystyle{plain}
\bibliography{refs} 

\newpage
\appendix
\section{Activation SQNR analysis across transformer layers}
\label{sec:act-sqnr-appendix}
\begin{figure}[ht!]
    \centering
    \begin{subfigure}{\textwidth}
        \centering
        \includegraphics[scale=0.64]{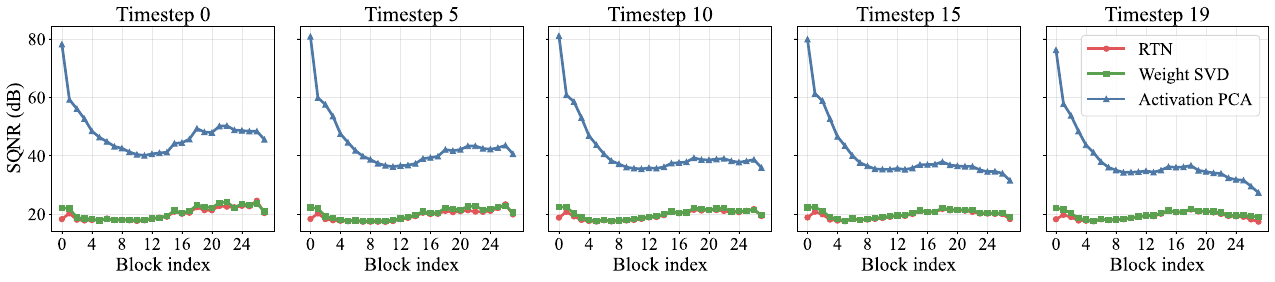}
        
        \vspace*{-1mm}\caption{Cross-attention query projection}
    \end{subfigure}
    \begin{subfigure}{\textwidth}
        \centering
        \includegraphics[scale=0.64]{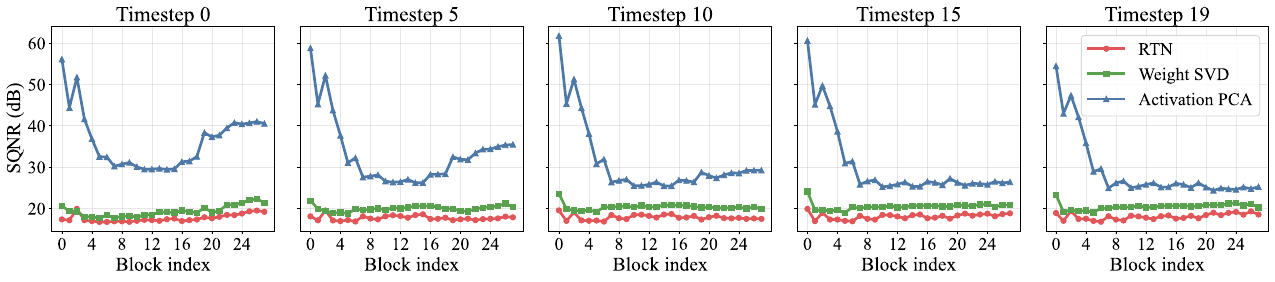}
        
        \vspace*{-1mm}\caption{Self-attention Q/K/V projections}
    \end{subfigure}
    \begin{subfigure}{\textwidth}
        \centering
        \includegraphics[scale=0.64]{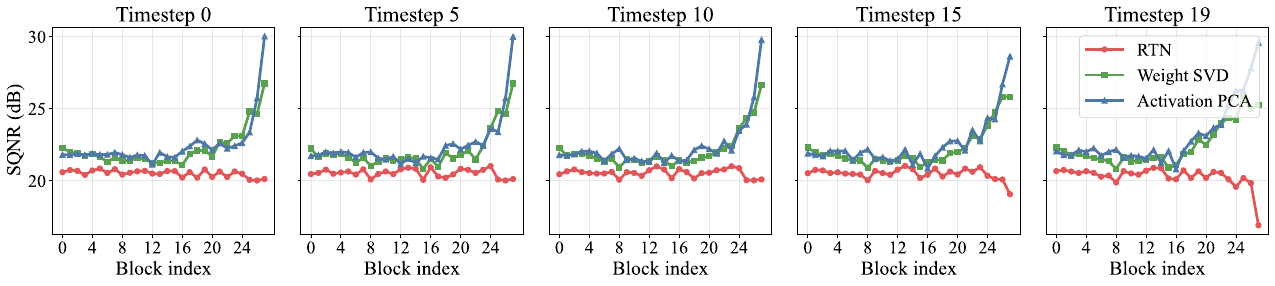}
        
        \vspace*{-1mm}\caption{Cross-attention output projection}
    \end{subfigure}
    \begin{subfigure}{\textwidth}
        \centering
        \includegraphics[scale=0.64]{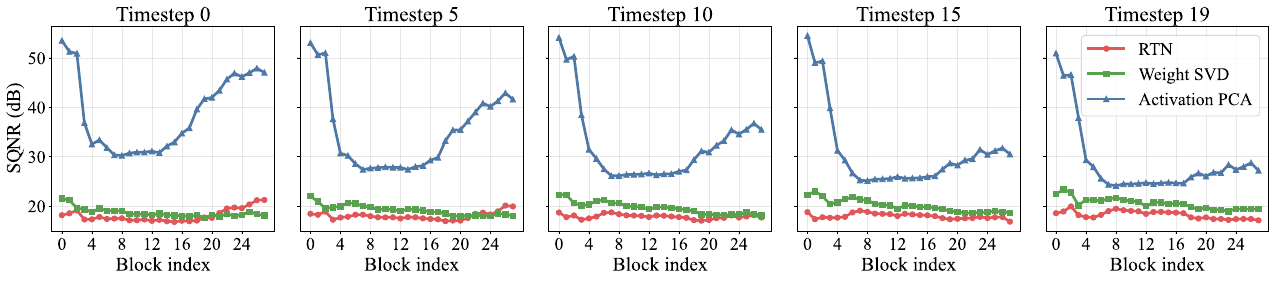}
        
        \vspace*{-1mm}\caption{FFN up projection}
    \end{subfigure}
    \begin{subfigure}{\textwidth}
        \centering
        \includegraphics[scale=0.64]{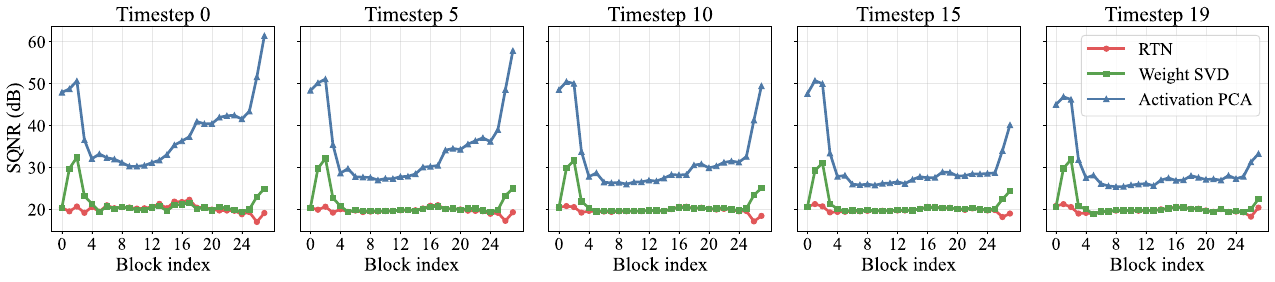}
        
        \vspace*{-1mm}\caption{FFN down projection}
    \end{subfigure}
    
    \vspace*{-1mm}\caption{Per-block activation SQNR (dB) for RTN, Weight SVD, and Activation PCA across the $28$ transformer blocks of PixArt-$\Sigma$, at five denoising timesteps.}
    \label{fig:sqnr-appendix}
    \vspace{-1mm}
\end{figure}

Figure~\ref{fig:sqnr-appendix} shows per-block activation SQNR across all $28$ transformer blocks of PixArt-$\Sigma$ for the following layers: \textbf{a)} \textit{cross-attention query projection}, \textbf{b)} \textit{self-attention Q/K/V projections}, \textbf{c)} \textit{cross-attention output projection}, and \textbf{d-e)} \textit{FFN up/down projections}. We compare two subspace-selection criteria for the high-precision branch: \textit{Weight SVD} (top singular vectors of $W$, the criterion used by SVDQuant), and \textit{Activation PCA} (top eigenvectors of $\mathbb{E}[xx^\top]$, the criterion used by DiRotQ), against the \textit{RTN} floor, which uses plain group-wise INT4 with no rotation and no high-precision branch. 
Activation PCA outperforms RTN across all layer types and timesteps, and outperforms Weight SVD on all layer types except the cross-attention output projection, where the two perform comparably. The largest gains appear in the cross-attention query projection, exceeding $20$~dB in early blocks. Weight SVD offers a modest but consistent improvement over RTN, yet falls substantially short of Activation PCA across all other layer types.


\section{PCA-Based activation quantization and quantization error analysis}
\label{sec:pca_details}
Using a calibration set of $128$ samples from \texttt{COCO} \texttt{Captions}~\citep{chen2015microsoft} and running PixArt-$\Sigma$~\citep{chen2024pixart} for $20$ denoising timesteps per sample, we estimate the empirical second-moment matrix of the activations $\mathbf{X}$ and perform its eigendecomposition:
\begin{equation}
    \mathbf{\Sigma} = \frac{1}{N} \mathbf{X}^\top \mathbf{X} = \mathbf{U} \mathbf{\Lambda} \mathbf{U}^\top, \quad \mathbf{\Lambda} = \text{diag}(\lambda_1, \dots, \lambda_m), \quad \lambda_1 \geq \dots \geq \lambda_m
\end{equation}
Here $\boldsymbol{\Sigma}$ is the uncentered second-moment matrix, which equals the true covariance under zero-mean activations. This approximately holds for post-LayerNorm activations in DiTs.

Rotating activations into the PCA basis, $\hat{\mathbf{X}} = \mathbf{X} \mathbf{U}$, yields decorrelated channels with variance $\text{Var}(\hat{\mathbf{X}}_j) = \lambda_j$, effectively ordering components by importance. Then, we partition the rotated channels into a high-precision subspace (first $k$ channels, capturing the largest eigenvalues) and a low-precision subspace (remaining $m - k$ channels): 
\begin{equation}
\hat{\mathbf{X}} = [\underbrace{\hat{\mathbf{X}}_{1:k}}_{\text{high-prec } (b_h \text{-bit})}  \; \underbrace{\hat{\mathbf{X}}_{k+1:m}}_{\text{low-prec } (b_l \text{-bit})}] 
\end{equation}

The expected squared quantization error per token decomposes per channel as:
\begin{equation}
\frac{1}{N}\,\mathbb{E}\left[\|\hat{\mathbf{X}} - Q(\hat{\mathbf{X}})\|_F^2\right] = \sum_{j=1}^{k} \epsilon_j(b_h) + \sum_{j=k+1}^{m} \epsilon_j(b_l),
\end{equation}
where $\epsilon_j(b) = \frac{1}{N}\sum_{i=1}^{N} \mathbb{E}\big[(\hat{X}_{ij} - Q(\hat{X}_{ij}))^2\big]$ is the expected squared quantization error per token for the $j^{th}$ channel under $b$-bit uniform quantization, and $N$ is the number of calibration tokens. For a channel with variance $\lambda_j$, this error is proportional to $\lambda_j/4^b$~\citep{gersho2012vector}, reflecting that higher-variance channels incur more absolute error, and increasing the bit width $b$ exponentially reduces the error. In practice, for $k$ high-variance channels we set $b_h = 16$ and for the remaining channels we use $b_l = 4$.


\section{Proof of orthonormality of \texorpdfstring{$\mathbf{U}_l \mathbf{R}$}{UlR}}
\label{app:proof}

\begin{lemma}
Let $\mathbf{U}_l \in \mathbb{R}^{m \times (m-k)}$ have orthonormal columns and $\mathbf{R} \in \mathbb{R}^{(m-k)\times(m-k)}$ be an orthogonal matrix. Then the columns of the product $\mathbf{U}_l \mathbf{R}$ are orthonormal.
\end{lemma}

\begin{proof}
By definition, $\mathbf{U}_l$ has orthonormal columns:
\[
\mathbf{U}_l^\top \mathbf{U}_l = \mathbf{I}.
\]
Since $\mathbf{R}$ is orthogonal, we have:
\[
\mathbf{R}^\top \mathbf{R} = \mathbf{I}.
\]

Consider the product $\mathbf{U}_l \mathbf{R}$. Its column orthonormality is verified as:
\[
(\mathbf{U}_l \mathbf{R})^\top (\mathbf{U}_l \mathbf{R}) 
= \mathbf{R}^\top \mathbf{U}_l^\top \mathbf{U}_l \mathbf{R} 
= \mathbf{R}^\top \mathbf{I} \mathbf{R} 
= \mathbf{R}^\top \mathbf{R} 
= \mathbf{I}.
\]

Hence, the columns of $\mathbf{U}_l \mathbf{R}$ are orthonormal.
\end{proof}

\section{Additional details on models, image generation settings, and datasets}
\label{app:model-generation-setting-dataset-details}
\paragraph{Models}
\texttt{PixArt-$\Sigma$}~\citep{chen2024pixart} is a compact DiT (600M parameters) optimized for high-resolution (up to 4K) generation. It employs stacked transformer blocks with self-attention, cross-attention, and feed-forward layers, along with a key\texttt{-}value token compression mechanism to improve efficiency. 
\texttt{FLUX.1-dev}~\citep{flux2024} is a guidance-distilled variant comprising 19 joint attention blocks~\citep{esser2024scaling} and 38 parallel attention blocks~\citep{dehghani2023scaling}, totaling 12B parameters. 
\texttt{SANA}~\citep{xie2025sana} is a 1.6B-parameter DiT that reduces token complexity via a $32\times$ compressed latent space and replaces softmax attention with linear attention for improved scalability. Licenses for all models are listed in Appendix~\ref{app:licenses}.

\paragraph{Image generation settings}
For all models, we generate images at a resolution of $1024\times 1024$ and use the \texttt{hash\_str\_to\_int} utility from Lhotse~\citep{lhotse_mixed_cut_docs} to derive deterministic per-prompt seeds for reproducibility. We employ the following model-specific generation settings:
\begin{itemize}[nosep, wide=0pt, leftmargin=*]
    \item \texttt{PixArt-$\Sigma$~\citep{chen2024pixart}}. We apply Classifier-Free Guidance (CFG)~\citep{ho2022classifier} with a guidance scale of $4.5$, balancing image fidelity and text alignment by strengthening the conditional signal relative to the unconditional branch~\citep{ho2022classifier}, and run $20$ inference steps.
    \item \texttt{FLUX.1-dev~\citep{flux2024}}. We use a distilled guidance scale of $3.5$ and perform all generations with $25$ sampling steps using the \texttt{FluxPipeline} from the Hugging Face Diffusers library~\citep{von-platen-etal-2022-diffusers}. We evaluate on the \texttt{dev} variant, which employs single-branch guidance (i.e., not dual-branch CFG), with a maximum sequence length of $512$.
    \item \texttt{SANA~\citep{xie2025sana}}. We use $20$ sampling steps with a CFG scale of $4.5$. Sampling is performed using the Flow-DPM scheduler~\citep{lu2022dpm}, which enables efficient and stable diffusion sampling.
\end{itemize}

\paragraph{Datasets}
MJHQ-30K~\citep{li2024playground} contains 30K prompts sourced from Midjourney~\citep{midjourney_ai}, spanning 10 categories (3K samples each) with diverse artistic styles. 
DCI~\citep{urbanek2024picture} contains approximately 8K images with densely annotated captions, often spanning thousands of words; we use its summarized variant (sDCI), where captions are reduced to 77 tokens using large language models to align with diffusion model input constraints. 
The corresponding licenses for all datasets are provided in Appendix~\ref{app:licenses}.

\vspace*{-1.5mm}
\section{Additional evaluation metrics}
\vspace*{-1.5mm}
\label{app:additiona-eval-metrics}
We provide three additional evaluation metrics to further assess semantic alignment and perceptual image quality:
\begin{itemize}[nosep, wide=0pt, leftmargin=*]
    \item \texttt{CLIP Score~\citep{hessel2021clipscore}}. Measures semantic alignment between a generated image and its text prompt using cosine similarity in the \texttt{CLIP}~\citep{radford2021learning} embedding space. Higher scores indicate better text-image alignment.
    \item \texttt{CLIP IQA~\citep{wang2023exploring}}. Evaluates perceptual image quality using \texttt{CLIP}~\citep{radford2021learning} by comparing images against quality-related textual prompts. It captures artifacts such as blur and noise, with higher scores indicating better visual quality.
    \item \texttt{SSIM~\citep{wang2004image}}. Measures structural similarity between images by comparing luminance, contrast, and structure. Higher values indicate greater similarity.
\end{itemize}

\vspace*{-1.5mm}
\section{Additional implementation details}
\vspace*{-1.5mm}
\label{app:additional_implementation_details}

\vspace*{-1.5mm}
\subsection{Quantization details}
\vspace*{-1.5mm}
\label{app:addtional_quantization_details}
Following SVDQuant~\citep{li2025svdquant}, we quantize both weights and activations of all linear layers to $4$-bit precision, except for the key and value projections in cross-attention, as they account for less than $5\%$ of total runtime. Additionally, input activations to adaptive normalization layers in the FLUX.1-dev model, and the down-projection layers in PixArt-$\Sigma$, are kept in higher precision (W4A16).

We target a high-precision fraction of $r{\approx}10\%$. Since the number of retained components, $k\texttt{=}rm$, is rounded to the nearest multiple of 64 for hardware efficiency, the realized fraction varies between $8$\texttt{–}$12\%$ across layers. We retain the top-$k$ high-variance activation components identified by PCA in high precision, while quantizing the remaining $88$\texttt{–}$92\%$ as residuals. The residual rotation $R$ is sampled once per residual dimension via QR decomposition of a Gaussian random matrix with a fixed seed of $42$.
The residual weights are quantized using GPTQ~\citep{frantar2023gptq}, for which we use a \texttt{damping} \texttt{factor} of \texttt{$ 0.01$}, following the original work, and a block size of $128$; the former stabilizes the Hessian inversion, while the latter controls the granularity of column-wise quantization, balancing accuracy and efficiency. To compute the PCA projection matrices and perform GPTQ calibration, we use $128$ randomly selected samples from COCO Captions~\citep{chen2015microsoft}. 
The entire pipeline is executed on a single NVIDIA A100 GPU.

\vspace*{-1.5mm}
\subsection{VLM-as-a-Judge evaluation details}
\vspace*{-1.5mm}
\label{app:vlm-judge}

Standard metrics (e.g., FID, LPIPS, CLIP Score) capture only isolated aspects of image quality and fail to jointly reflect visual coherence and prompt fidelity, particularly under quantization where degradations may appear as subtle artifacts or semantic drift. To address this, VLM-as-a-Judge evaluation has gained traction as a more holistic alternative~\citep{chen2024mj, hayes2025finegrain}, with recent work showing that multimodal language models provide reliable, fine-grained assessments aligned with human judgments~\citep{chen2024mllm, lee2024prometheus}.

In this work, we adopt a VLM-as-a-Judge evaluation protocol inspired by VIEScore~\citep{ku2024viescore}, which uses MLLMs as instruction-guided evaluators for conditional image generation. VIEScore reports that GPT-4o~\citep{openai2024gpt4o} judgments achieve a Spearman correlation of $\rho\texttt{=}0.4$ with human judgments, approaching human-to-human agreement of $\rho\texttt{=0.45}$~\citep{ku2024viescore}. It evaluates images along two complementary axes: Semantic Consistency (SC), measuring prompt faithfulness, and Perceptual Quality (PQ), assessing visual coherence, sharpness, and artifacts, with the overall score computed as their geometric mean. Since GPT-4o is proprietary and costly to run at the scale our evaluation requires, we follow this two-axis structure using three independent judges: two open-source models, \texttt{Qwen2-VL-7B-Instruct}~\citep{wang2024qwen2} and \texttt{InternVL2-8B}~\citep{chen2024expanding}, and one closed-source model, \texttt{Gemini~3.6~Flash}~\citep{google2026gemini36flash}, reporting results from all three. Qwen2-VL provides holistic image-level assessment, while InternVL2's dynamic tiling enables fine-grained local inspection, offering complementary coverage of global and local generation quality; the closed-source Gemini 3.6 Flash, from a different model family, further verifies that our conclusions are robust to the choice of the judge.

\paragraph{Evaluation setup}
For each method (DiRotQ and SVDQuant) and quantization format (INT4 and NVFP4), we evaluate 5K images generated by PixArt-$\Sigma$ on MJHQ-30K, spanning all 10 prompt categories. Each image is scored three times with low-temperature sampling (\texttt{do\_sample=True}, \texttt{temperature=0.3}), and the numeric scores are averaged across runs to reduce stochastic variance. Each image is scored via an independent query using the prompt below. As the judge evaluates one image at a time and never sees outputs from other methods in the same context, position bias across methods does not arise. Following VIEScore, the judge produces a step-by-step rationale \textit{before} assigning scores~\citep{zheng2023judging}, improving alignment with human judgments. The enforced overall score is computed as $\text{Overall} {=} \sqrt{\text{SC} \times \text{PQ}}$, ensuring that failure on either dimension cannot be masked by the other. Differences below $0.01$ are treated as ties.

\paragraph{Judge prompt}
\label{app:judge_prompt}
The judge receives the text prompt and the generated image. The prompt instructs the model to return a JSON object with the \texttt{rationale} key \textit{first} (think step-by-step before scoring), followed by \texttt{Semantic\_Consistency}, \texttt{Perceptual\_Quality}, and \texttt{Overall} on a $1\texttt{-}10$ scale, with rubric anchors at every even level. The overall score is explicitly instructed to weight both dimensions equally by taking their geometric mean.

\begin{tcolorbox}[
  colback=gray!5, colframe=gray!35,
  fontupper=\scriptsize\ttfamily,
  left=4pt, right=4pt, top=4pt, bottom=4pt,
  boxsep=0pt
]
A text-to-image model generated this image from the prompt: \{prompt\}\\
Evaluate this image and return ONLY a JSON object with "rationale"
key FIRST (reason step-by-step before scoring), then:\\
\{"rationale": "<describe what you see, prompt alignment, artifacts>",\\
\ "Semantic\_Consistency": <1-10>,\\
\ "Perceptual\_Quality": <1-10>,\\
\ "Overall": <1-10>\}\\
SC: 10=all details match, 8=minor omissions, 6=key details wrong, 4=loosely related, 2=no relation, 1=incomprehensible.\\
PQ: 10=photorealistic/no artifacts, 8=minor artifacts, 6=noticeable distortion, 4=significant artifacts, 2=severe, 1=broken.\\
Overall: weight both terms equally by taking their geometric mean.

\end{tcolorbox}


\begin{table*}[t!]
\centering
\caption{Comparison of CLIP IQA, CLIP Score (SCR), and SSIM on the MJHQ-30K and sDCI datasets across different models. All results are obtained using the official codebases of each method. All methods, except PTQ4DiT and ViDiT-Q, employ GPTQ~\citep{frantar2023gptq} for weight quantization. $\uparrow$ indicates higher is better. DiRotQ substantially improves both quality and similarity metrics across different settings, reinforcing its advantage at low-bit precision.}
\label{tab:visual-quality-results-appendix}
\resizebox{0.99\textwidth}{!}{
\begin{tabular}{clccccccc}
\toprule 
 \multirow{3}{*}[-1.0ex]{\textbf{Model}} & \multirow{3}{*}[-1.0ex]{\textbf{Precision}} & \multirow{3}{*}[-1.0ex]{\textbf{Method}} & \multicolumn{3}{c}{\textbf{MJHQ-30K}} & \multicolumn{3}{c}{\textbf{sDCI}}  \\ \cmidrule(r){4-6}  \cmidrule(l){7-9} 
 & & & \multicolumn{2}{c}{\textbf{Quality}} &  \textbf{Similarity} & \multicolumn{2}{c}{\textbf{Quality}} &  \textbf{Similarity} \\ \cmidrule(r){4-5}  \cmidrule(lr){6-6}  \cmidrule(lr){7-8}  \cmidrule(l){9-9} 
  & & & \textbf{CLIP IQA ($\uparrow$)} & \textbf{CLIP SCR ($\uparrow$)} & \textbf{SSIM ($\uparrow$)} & \textbf{CLIP IQA ($\uparrow$)} & \textbf{CLIP SCR ($\uparrow$)} & \textbf{SSIM ($\uparrow$)}  \\ \midrule \midrule 
  & FP16     & - & 0.944 & 26.8 & - & 0.966 & 26.1 & - \\ \cmidrule{2-9}
  & PTQ4DiT  & INT W4A8 & 0.899 & 25.9 & 0.353 & 0.911 & 25.5 & 0.297 \\
  & Q-DiT    & INT W4A8 & \textbf{0.948} & 26.4 & 0.493 & \textbf{0.967} & 25.9 & 0.421\\
  & \cellcolor[HTML]{e3f4f7}DiRotQ & \cellcolor[HTML]{e3f4f7}INT W4A8 & \cellcolor[HTML]{e3f4f7}0.943 & \cellcolor[HTML]{e3f4f7}\textbf{26.8} & \cellcolor[HTML]{e3f4f7}\textbf{0.754} & \cellcolor[HTML]{e3f4f7}0.965 & \cellcolor[HTML]{e3f4f7}\textbf{26.1} & \cellcolor[HTML]{e3f4f7}\textbf{0.701} \\ \cmidrule{2-9}
  & ViDiT-Q  & INT W4A4 & 0.914 & \textbf{26.7} & 0.345 & 0.945 & \textbf{26.4} & 0.268 \\
  & SVDQuant & INT W4A4 & 0.928 & \textbf{26.7} & 0.656 & 0.949 & 26.0 & 0.574 \\
   & \cellcolor[HTML]{e3f4f7}DiRotQ & \cellcolor[HTML]{e3f4f7}INT W4A4 & \cellcolor[HTML]{e3f4f7}\textbf{0.944} & \cellcolor[HTML]{e3f4f7}\textbf{26.7} & \cellcolor[HTML]{e3f4f7}\textbf{0.727} & \cellcolor[HTML]{e3f4f7}\textbf{0.963} & \cellcolor[HTML]{e3f4f7}26.0 & \cellcolor[HTML]{e3f4f7}\textbf{0.667} \\ \cmidrule{2-9}
  \multirow{-9}{*}{PixArt-$\Sigma$} & SVDQuant & NVFP W4A4 & 0.939 & \textbf{26.7} & 0.699 & 0.957 & \textbf{26.1} & 0.628 \\
  \multirow{-9}{*}{(20 Steps)} & \cellcolor[HTML]{e3f4f7}DiRotQ & \cellcolor[HTML]{e3f4f7}NVFP W4A4 & \cellcolor[HTML]{e3f4f7}\textbf{0.943} & \cellcolor[HTML]{e3f4f7}\textbf{26.7} & \cellcolor[HTML]{e3f4f7}\textbf{0.736} & \cellcolor[HTML]{e3f4f7}\textbf{0.962} & \cellcolor[HTML]{e3f4f7}26.0 & \cellcolor[HTML]{e3f4f7}\textbf{0.676} \\ \midrule \midrule

  & BF16     & - & 0.946 & 26.0 & - & 0.951 & 25.5 & - \\ \cmidrule{2-9}
  & Q-DiT    & INT W4A8 &  0.942 & 25.9 & 0.810 & 0.946 & \textbf{25.5} & 0.751 \\
  & \cellcolor[HTML]{e3f4f7}DiRotQ   & \cellcolor[HTML]{e3f4f7}INT W4A8 & \cellcolor[HTML]{e3f4f7}\textbf{0.946} & \cellcolor[HTML]{e3f4f7}\textbf{26.0} & \cellcolor[HTML]{e3f4f7}\textbf{0.870} & \cellcolor[HTML]{e3f4f7}\textbf{0.951} & \cellcolor[HTML]{e3f4f7}\textbf{25.5} & \cellcolor[HTML]{e3f4f7}\textbf{0.833} \\ \cmidrule{2-9}
  & ViDiT-Q  & INT W4A4 & 0.914 & 22.6 & 0.470 & 0.897 & 25.2 & 0.427 \\
  & SVDQuant & INT W4A4 & \textbf{0.946} & 25.7 & 0.789 & 0.948 & 25.4 & 0.726 \\
  & \cellcolor[HTML]{e3f4f7}DiRotQ & \cellcolor[HTML]{e3f4f7}INT W4A4 & \cellcolor[HTML]{e3f4f7}0.945 & \cellcolor[HTML]{e3f4f7}\textbf{25.9} & \cellcolor[HTML]{e3f4f7}\textbf{0.865} & \cellcolor[HTML]{e3f4f7}\textbf{0.951} & \cellcolor[HTML]{e3f4f7}\textbf{25.5} & \cellcolor[HTML]{e3f4f7}\textbf{0.824} \\ \cmidrule{2-9}
  \multirow{-9}{*}{FLUX.1-dev} & SVDQuant & NVFP W4A4 & \textbf{0.947} & \textbf{25.9} & 0.833 & \textbf{0.951} & 25.4 & 0.784 \\
  \multirow{-9}{*}{(25 Steps)} & \cellcolor[HTML]{e3f4f7}DiRotQ   & \cellcolor[HTML]{e3f4f7}NVFP W4A4 & \cellcolor[HTML]{e3f4f7}\textbf{0.947} & \cellcolor[HTML]{e3f4f7}\textbf{25.9} & \cellcolor[HTML]{e3f4f7}\textbf{0.862} & \cellcolor[HTML]{e3f4f7}\textbf{0.951} & \cellcolor[HTML]{e3f4f7}\textbf{25.5} & \cellcolor[HTML]{e3f4f7}\textbf{0.823} \\ \midrule \midrule

  & BF16     & - & 0.941 & 27.3 & - & 0.967 & 26.7 & - \\ \cmidrule{2-9}
  & PTQ4DiT  & INT W4A8 & 0.876 & 26.2 & 0.315 & 0.932 & 25.9 & 0.246 \\
  & Q-DiT    & INT W4A8 & \textbf{0.943} & \textbf{27.3} & 0.666 & \textbf{0.967} & \textbf{26.8} & 0.594 \\
  & \cellcolor[HTML]{e3f4f7}DiRotQ   & \cellcolor[HTML]{e3f4f7}INT W4A8 & \cellcolor[HTML]{e3f4f7}0.940 & \cellcolor[HTML]{e3f4f7}\textbf{27.3} & \cellcolor[HTML]{e3f4f7}\textbf{0.867} & \cellcolor[HTML]{e3f4f7}\textbf{0.967} & \cellcolor[HTML]{e3f4f7}26.7 & \cellcolor[HTML]{e3f4f7}\textbf{0.825} \\ \cmidrule{2-9}
  & ViDiT-Q  & INT W4A4 & 0.938 & \textbf{27.3} & 0.669 & 0.965 & \textbf{26.8} & 0.598 \\
  & SVDQuant & INT W4A4 & 0.932 & \textbf{27.3} & 0.671 & 0.961 & 26.6 & 0.617 \\
  & \cellcolor[HTML]{e3f4f7}DiRotQ & \cellcolor[HTML]{e3f4f7}INT W4A4 & \cellcolor[HTML]{e3f4f7}\textbf{0.940} & \cellcolor[HTML]{e3f4f7}\textbf{27.3} & \cellcolor[HTML]{e3f4f7}\textbf{0.854} & \cellcolor[HTML]{e3f4f7}\textbf{0.967} & \cellcolor[HTML]{e3f4f7}26.7 & \cellcolor[HTML]{e3f4f7}\textbf{0.813} \\ \cmidrule{2-9}
  \multirow{-9}{*}{SANA-1.6B} & SVDQuant & NVFP W4A4 & 0.939 & \textbf{27.3} & 0.744 & 0.965 & \textbf{26.7} & 0.688 \\
  \multirow{-9}{*}{(20 Steps)} & \cellcolor[HTML]{e3f4f7}DiRotQ & \cellcolor[HTML]{e3f4f7}NVFP W4A4 & \cellcolor[HTML]{e3f4f7}\textbf{0.941} & \cellcolor[HTML]{e3f4f7}\textbf{27.3} & \cellcolor[HTML]{e3f4f7}\textbf{0.846} & \cellcolor[HTML]{e3f4f7}\textbf{0.967} & \cellcolor[HTML]{e3f4f7}\textbf{26.7} & \cellcolor[HTML]{e3f4f7}\textbf{0.805} \\ \bottomrule
  \end{tabular}
}
\end{table*}

\section{Additional results}
\subsection{Visual quality results}
\label{app:visual_quality_results_appendix}
Table~\ref{tab:visual-quality-results-appendix} reports additional perceptual and semantic metrics, including CLIP Score, CLIP IQA, and SSIM. Since SSIM is computed against 16-bit reference outputs, it is undefined for reference rows and is marked as ``--'' in Table~\ref{tab:visual-quality-results-appendix}. On the MJHQ-30K dataset, DiRotQ substantially improves both perceptual quality and similarity: for PixArt-$\Sigma$ at INT W4A4, it improves SSIM from $0.656$ (SVDQuant) to $0.727$ and CLIP IQA from $0.928$ to $0.944$, while maintaining the same CLIP Score ($26.7$). Similar trends hold for NVFP4, where SSIM increases from $0.699$ to $0.736$. For FLUX.1-dev, DiRotQ achieves substantial gains in SSIM ($0.865$ vs. $0.789$ for INT W4A4) while maintaining or slightly improving CLIP-based metrics. For SANA-1.6B, the improvements are even more pronounced, with SSIM increasing from $0.671$ to $0.854$ (INT W4A4) and from $0.744$ to $0.846$ (NVFP4). On the sDCI dataset, DiRotQ shows consistent improvements in SSIM and maintains competitive CLIP IQA and CLIP Score across models, further confirming its effectiveness under $4$-bit quantization. Figures~\ref{fig:pixart_visual},~\ref{fig:flux_visual}, and~\ref{fig:sana_visual} visualize more qualitative comparisons for PixArt-$\Sigma$, FLUX.1-dev, and SANA-1.6B, respectively. 

\begin{table}[t]
\centering
\vspace*{-2mm}

\caption{FLUX.1-dev forward-pass latency breakdown on RTX 4090
(DiRotQ INT W4A4).}
\label{tab:latency_breakdown}
\resizebox{0.99\textwidth}{!}{
\begin{tabular}{lccccccc}
\toprule
\textbf{Batch} & \textbf{INT4} & \textbf{Online} & \textbf{Norm/}
  & \textbf{BF16} & \textbf{BF16} & \textbf{Act.} & \multirow{2}{*}{\textbf{Other}} \\
\multicolumn{1}{c}{\textbf{size}} & \textbf{GEMM} & \textbf{rotation} & \textbf{elementwise}
  & \textbf{HP branch} & \textbf{softmax attn.} &  \textbf{quant/dequant} & \\
\midrule
$B{=}1$ & 31.7\% & 19.1\% & 8.1\%  & 2.8\% & 35.3\% & 2.2\% & 0.8\%\\
$B{=}2$ & 29.4\% & 16.7\% & 14.0\% & 1.7\% & 35.1\% & 2.3\% & 0.8\%\\
\bottomrule
\end{tabular}

\vspace*{-2mm}
}
\vspace*{-3mm}
\end{table}

\vspace*{-1mm}
\subsection{Inference latency breakdown for DiRotQ INT W4A4}
\vspace*{-1mm}
\label{app:inference_latency_breakdown}
For FLUX.1-dev~\cite{flux2024} ($11.90$B parameters), $419$ linear layers, comprising $8.628$B parameters ($72.5\%$ of the model), are quantized to INT4, while $85$ layers, accounting for $3.274$B parameters ($27.5$\%), remain in BF16. The excluded BF16 weights are predominantly in AdaLN modulation layers, which account for $98.7\%$ of the excluded parameters. Input embedders and the final \texttt{proj\_out} layer contribute less than $0.4\%$ of the total parameters. FLUX.1-dev employs a multimodal DiT architecture with joint image--text self-attention. Accordingly, all Q/K/V/O projections are quantized to INT4, while softmax attention remains in BF16, following SVDQuant~\citep{li2025svdquant}, as attention scores and softmax are more sensitive to low-bit quantization than linear projections. For latency and memory measurements, no W4A16 layers are used: all excluded modules, including the AdaLN layers, remain entirely in BF16 (W16A16). This is more conservative than the W4A16 setting used for quality evaluation (Appendix~\ref{app:addtional_quantization_details}), since keeping these weights in BF16 increases memory and compute.

Table~\ref{tab:latency_breakdown} breaks down forward-pass latency on the RTX 4090 across batch sizes $1$ and $2$. INT4 GEMMs account for ${\sim}30\%$ of runtime and provide the primary performance gain, while the online rotation contributes $17$-$19\%$ and activation quant/dequant remains small at ${\sim2}\%$. BF16 softmax attention dominates the higher-precision runtime (${\sim}35\%$); the excluded BF16 linear layers contribute only $1.7$-$2.8\%$, since AdaLN modulation is small relative to the main token-processing path. The relative cost of the online rotation is nearly constant across batch sizes because both rotation and INT4 GEMMs scale with the same token dimension, so the rotation is not expected to become a bottleneck as batch size or image resolution increases.

\vspace{-2mm}
\subsection{NVFP4 memory saving and speedup}
\vspace{-1mm}
\label{app:NVFP4_memory_savings_and_speedup}
Figure~\ref{fig:NVFP4_memory_savings_and_speedup} shows the memory savings and inference speedup of DiRotQ under NVFP4 on FLUX.1-dev, evaluated on an RTX PRO 6000 Blackwell GPU. At batch size $1$, DiRotQ achieves a $2.0\times$ speedup and $2.0\times$ memory reduction over the BF16 baseline, with comparable gains at batch size $2$. Unlike the RTX 4090 setting, the BF16 baseline does not require CPU offloading even at batch size $2$, as the GPU's 96~GB memory is sufficient to accommodate the full model. Thus, the measured speedup reflects the compute and memory-bandwidth benefits of 4-bit inference, without the additional gains from avoiding CPU offloading.

\begin{figure}[t!]
    \centering
    \includegraphics[width=0.98\linewidth]{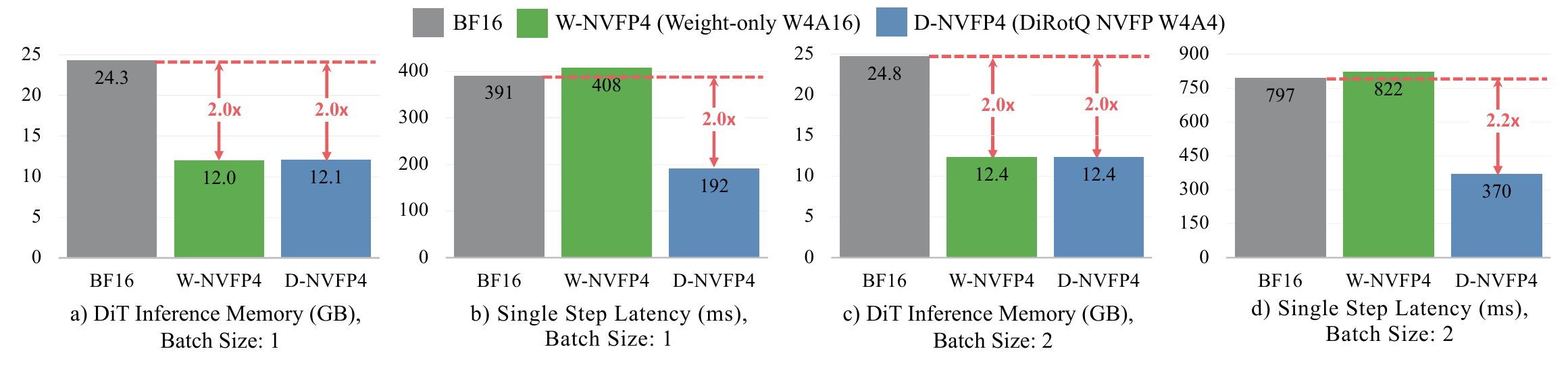}
    \vspace*{-2mm}
    
    \caption{Memory usage and one-step latency for FLUX.1-dev under NVFP4 across batch sizes 1 and 2, on an RTX PRO 6000 Blackwell GPU. DiRotQ achieves a $2.0\times$ speedup and $2.0\times$ memory reduction over the BF16 baseline at batch size 1, with comparable gains at batch size 2.}
    \label{fig:NVFP4_memory_savings_and_speedup}
    \vspace*{-3mm}
\end{figure}

\begin{figure}[ht!]
    \centering
    \begin{subfigure}{\textwidth}
        \centering
        \includegraphics[scale=0.488]{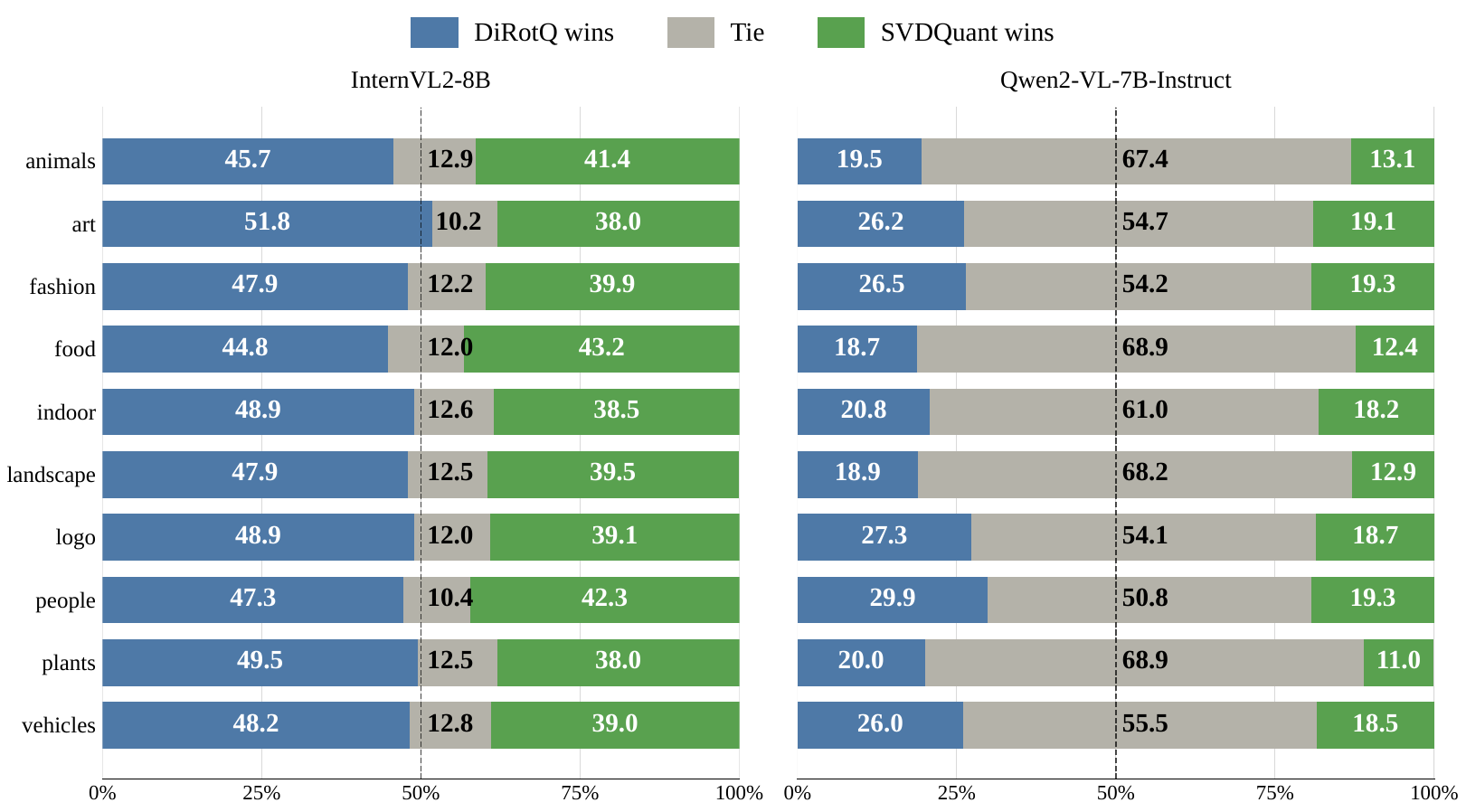}
        
        \vspace*{-2mm} \caption{INT4 quantization.}
        \label{fig:vlm_judge_int4_per_category}
    \end{subfigure}
    
     \vspace*{1mm} 
    \begin{subfigure}{\textwidth}
        \centering
        \includegraphics[scale=0.488]{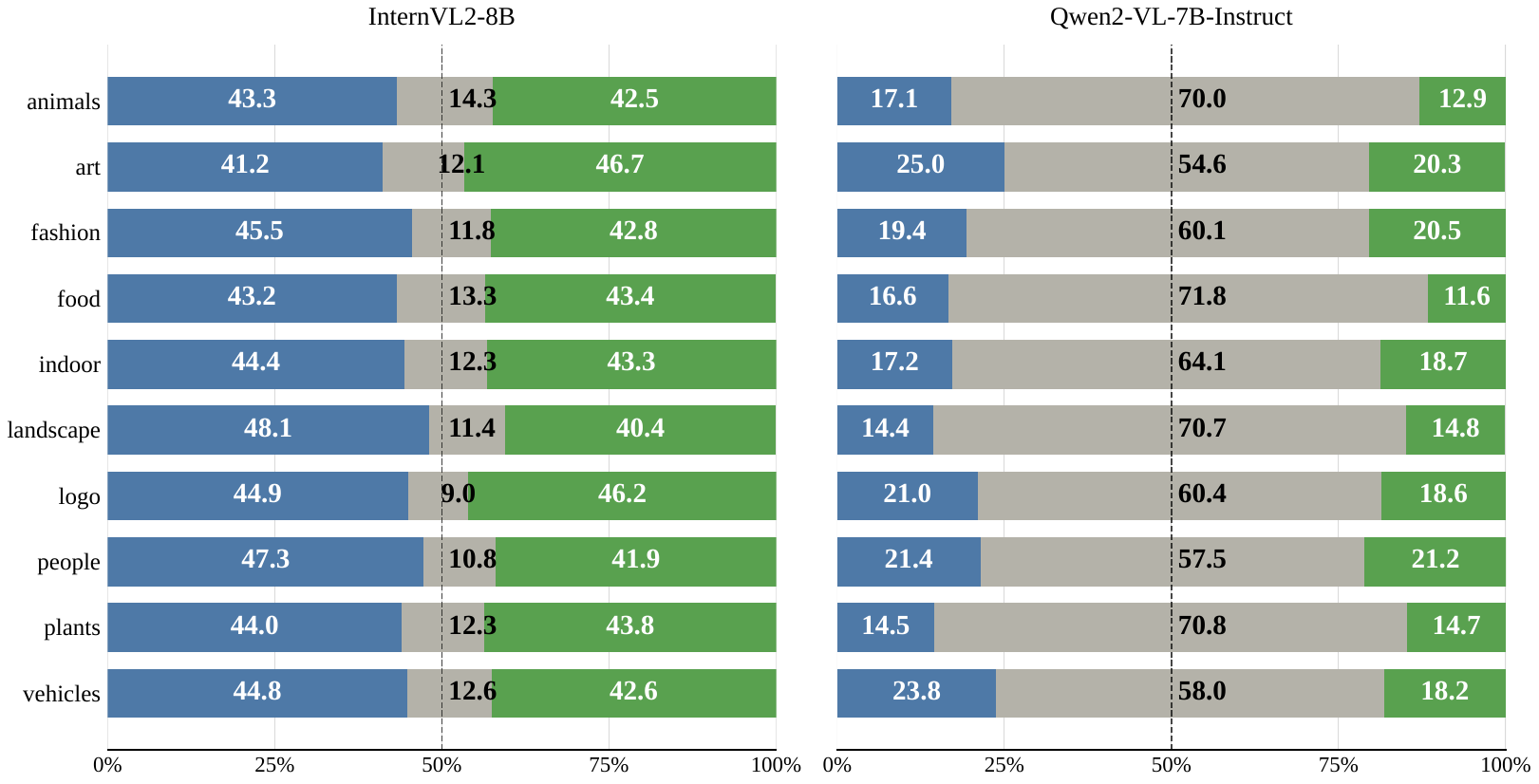}
        
        \vspace*{-2mm} \caption{NVFP4 quantization.}
        \label{fig:vlm_judge_nvfp4_per_category}
    \end{subfigure}
    
    \vspace*{-1mm}
    \caption{Per-category VLM-as-a-Judge pairwise comparison between SVDQuant and DiRotQ on the Overall score, for PixArt-$\Sigma$ generations across 10 MJHQ-30K categories, judged by two open-source models, InternVL2-8B (left) and Qwen2-VL-7B-Instruct (right). (a) INT4 quantization. (b) NVFP4 quantization. Dashed line indicates the $50\%$ parity point; differences below $0.01$ are treated as ties.}
    \label{fig:vlm_judge_per_category}
\end{figure}

\vspace{-1mm}
\subsection{VLM-as-a-Judge additional results}
\vspace{-1mm}
\label{app:vlm_as_judge_per_category}
\paragraph{Per-category results for open-source judges} Figure~\ref{fig:vlm_judge_per_category} breaks down the Overall-score win/tie/loss rates per MJHQ-30K category for two open-source judges, InternVL2-8B and Qwen2-VL-7B-Instruct. Under INT4 quantization (Figure~\ref{fig:vlm_judge_int4_per_category}), DiRotQ is preferred on every one of the 10 categories under both judges, with the largest gaps on stylized content such as art ($51.8\%$ vs. $38.0\%$ under InternVL2-8B) and people ($29.9\%$ vs. $19.3\%$ under Qwen2-VL-7B-Instruct). Under NVFP4 (Figure~\ref{fig:vlm_judge_nvfp4_per_category}), the per-category gaps narrow as both methods approach the FP16 baseline, with DiRotQ still preferred on $7$ of $10$ categories under InternVL2-8B and $6$ of $10$ under Qwen2-VL-7B-Instruct.

\vspace{-1mm}
\paragraph{Closed-source judge (Gemini 3.6 Flash) results} We repeat the pairwise evaluation with the closed-source Gemini~3.6~Flash~\citep{google2026gemini36flash} for the INT4 PixArt-$\Sigma$ / MJHQ-30K benchmark on a random subset of 1K prompts (100 per category), using the same judge prompt, temperature, three-run averaging, and geometric-mean Overall score as the open-source judges. Gemini~3.6~Flash prefers DiRotQ over SVDQuant on all three axes: Semantic Consistency $34.5\%$ vs.\ $24.3\%$ ($41.2\%$ ties), Perceptual Quality $55.7\%$ vs.\ $23.2\%$ ($21.1\%$ ties), and Overall $58.2\%$ vs.\ $29.9\%$ ($11.9\%$ ties). The margin is largest on Perceptual Quality, while Semantic Consistency carries the highest tie rate, indicating that prompt faithfulness is typically comparable between the two methods and that DiRotQ separates most clearly on perceptual fidelity, consistent with quantization primarily degrading visual quality rather than semantic content (Figure~\ref{fig:gemini_judge_per_category}). 

\begin{figure}[t!]
    \centering
    \includegraphics[width=\linewidth]{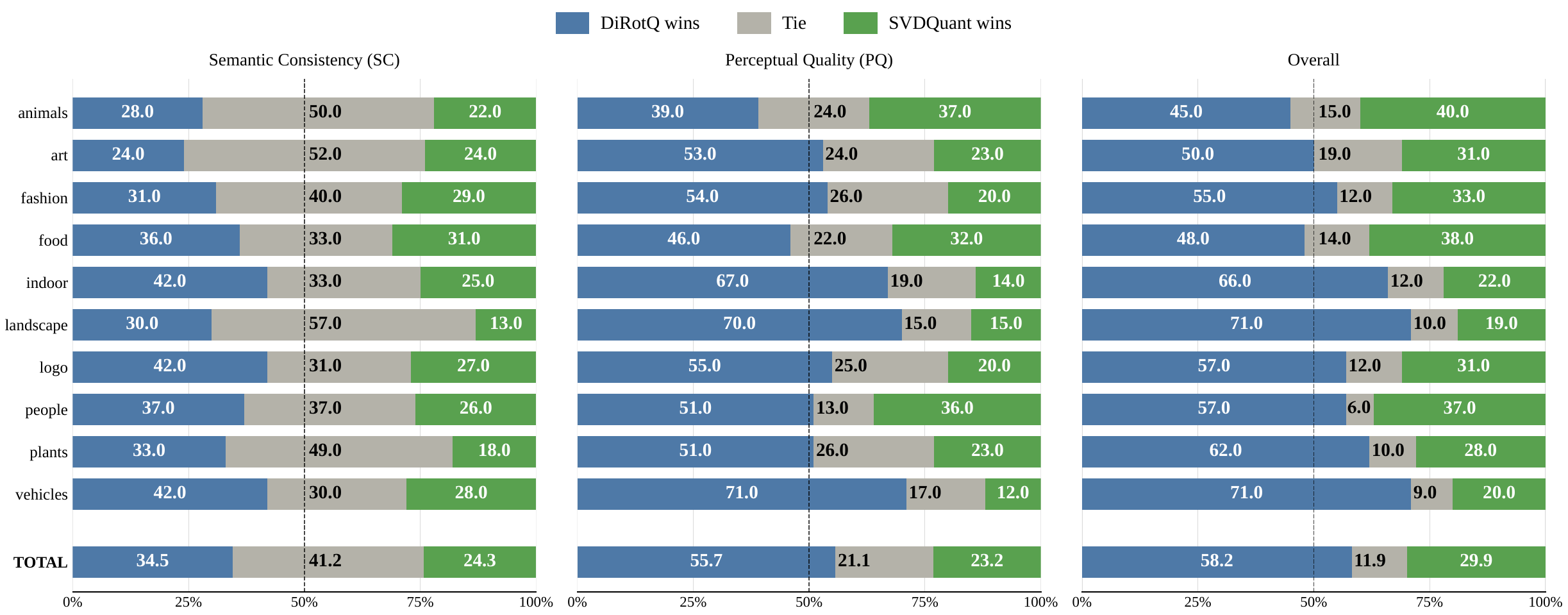}
    \caption{Per-category VLM-as-a-Judge pairwise comparison between DiRotQ and SVDQuant on PixArt-$\Sigma$ generations across 10 MJHQ-30K categories under INT4 quantization, judged by Gemini~3.6~Flash along Semantic Consistency (SC), Perceptual Quality (PQ), and the geometric-mean Overall score. The dashed line marks the 50\% parity point; differences below 0.01 are treated as ties. The bottom \textsc{total} row aggregates over all categories.}
    \label{fig:gemini_judge_per_category}
\end{figure}

\vspace{-1mm}
\subsection{High-precision tail fraction ablation}
\vspace{-1mm}
\label{app:hp_sweep}
A central design choice in DiRotQ is the rank ratio $r$, which controls the fraction of PCA-rotated activation channels kept in FP16, while the remaining $1\texttt{-}r$ are quantized to 4-bit. Increasing $r$ improves fidelity but also increases the memory footprint of the FP16 branch and the cost of online rotation, making it important to understand how performance scales with $r$. We sweep $r$ from $5\%$ to $25\%$ and report activation SQNR, averaged over 28 transformer blocks of PixArt-$\Sigma$ on MJHQ-30K, for FFN Up, Self-attention QKV, and Cross-attention Q layers. Results across five denoising timesteps are shown in Figure~\ref{fig:hp_sweep}, with RTN as a reference baseline.

Even with only a $5\%$ FP16 tail, DiRotQ already surpasses RTN by roughly $7$-$24$ dB across all layer types and timesteps, demonstrating that the PCA rotation itself, rather than the size of the high-precision tail, drives most of the quality improvement. Increasing $r$ to $10\%$ yields a further $2$-$3$ dB gain, while larger values provide diminishing returns ($\texttt{<}0.5$ dB per additional $5\%$ in most cases). This saturation arises because PCA concentrates most outlier variance into a small set of principal components, so even a modest FP16 tail captures the dominant quantization error. The trend is consistent across layers and timesteps, suggesting $r\texttt{=}10\%$ as a robust default that avoids per-layer tuning while recovering most of the gap to the full-precision baseline.

\subsection{Calibration data and timestep sensitivity ablation}
\label{app:ablation-calib} 
To measure the sensitivity of the PCA basis to the calibration prompt source, the number of calibration examples, and timestep coverage, we report per-layer activation SQNR on PixArt-$\Sigma$ at a fixed $10\%$ FP16 budget. For each configuration, the PCA basis is refit from scratch and evaluated on a common held-out MJHQ-30K set.

\paragraph{Sensitivity to calibration prompt source} The PCA basis is insensitive to the source of the calibration prompts. Fitting it on COCO-derived captions versus MJHQ captions, two datasets with markedly different content and style, at a fixed size of $N\texttt{=}128$ and evaluating both on the same held-out MJHQ set changes the mean activation SQNR by only $0.04$~dB ($27.17$ vs.\ $27.21$~dB). Notably, calibrating on the out-of-domain COCO prompts is essentially as effective as calibrating in-domain, indicating that the dominant activation directions the basis captures are a property of the model rather than of the specific calibration corpus.


\paragraph{Sensitivity to calibration sample size} Figure~\ref{fig:calib-size-sweep} plots activation SQNR against calibration set
size ($16{\to}512$) for the \emph{most calibration-sensitive} transformer block of each layer type (FFN up, self-attention QKV, and cross-attention Q), at every denoising timestep. The DiRotQ curves remain essentially flat across the full $32\times$ range and consistently outperform the RTN reference. Increasing the calibration set size from $16$ to $512$ samples changes the mean SQNR by only
$0.09$~dB, indicating that the estimated PCA basis is stable over a wide range of calibration set sizes. This small variation is substantially lower than the variation observed across layer types and denoising timesteps.

\paragraph{Sensitivity to calibration timestep coverage}
Table~\ref{tab:calib_timestep} reports mean activation SQNR aggregated across all layers and transformer blocks at selected points along the denoising trajectory. SQNR decreases from $33.6$~dB at timestep~0 to $27.8$~dB at timestep~19, resulting in a total spread of approximately $5.8$~dB. This variation indicates that activation statistics change substantially throughout denoising, making timestep coverage an important consideration when collecting calibration data for the shared PCA basis. Covering the full denoising trajectory allows the calibration set to capture these variations rather than fitting the rotation to a limited portion of the trajectory.

\begin{figure}[t]
    \centering
    \includegraphics[width=\linewidth]{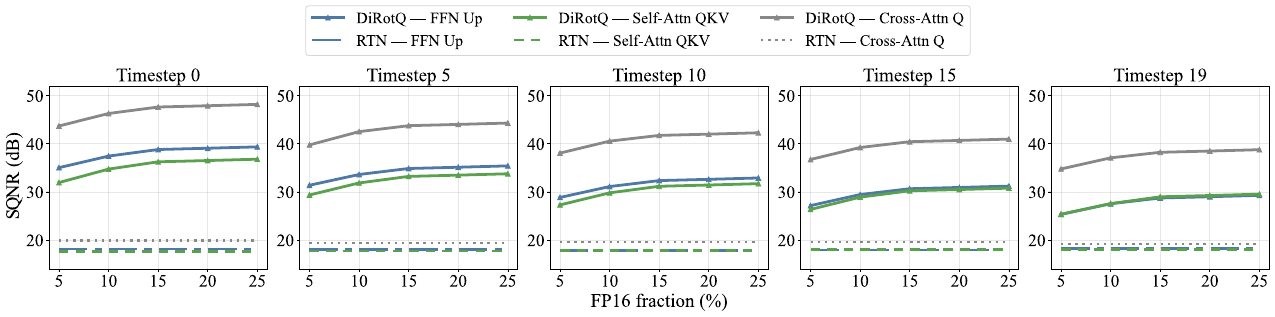}
    \caption{Activation SQNR (dB) on PixArt-$\Sigma$ as a function of FP16 tail fraction, averaged across all 28 transformer blocks, for FFN up, self-attention QKV, and cross-attention Q layers across five denoising timesteps. RTN baselines are shown as horizontal references. SQNR rises sharply up to $r\texttt{=}10\%$ and saturates beyond, motivating our choice of $r\texttt{=}10\%$ throughout the paper.}
    \label{fig:hp_sweep}
\end{figure}

\begin{figure}[t]
    \centering
    \includegraphics[width=\linewidth]{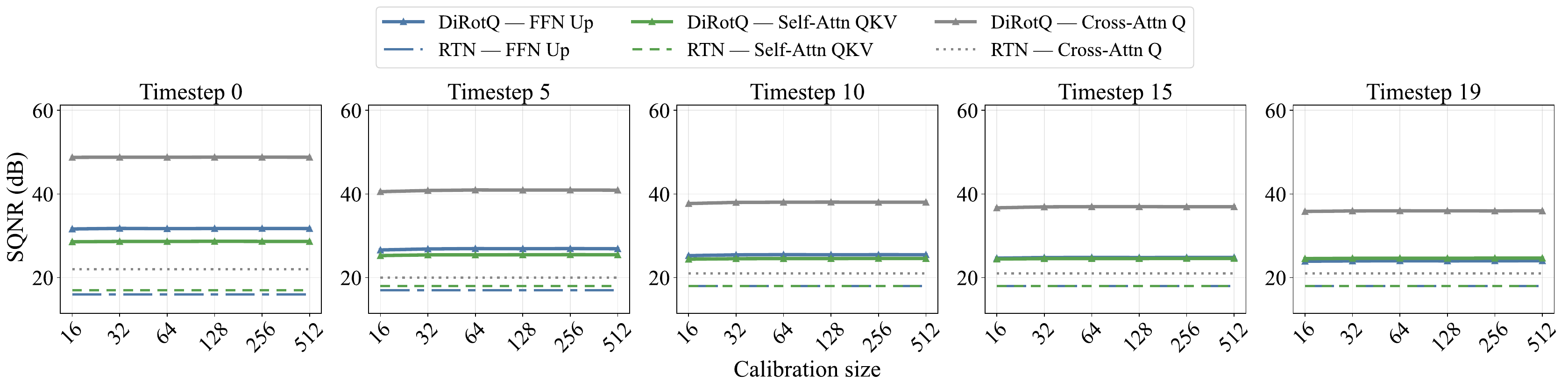}
    \caption{Effect of calibration set size. Activation SQNR versus calibration set size ($16 \to 512$, log scale) for the most calibration-sensitive block of three different layer types: FFN up, self-attention QKV, and cross-attention Q, shown across denoising timesteps. DiRotQ (solid) remains nearly constant over the $32\times$ range, while the RTN reference (dashed) remains substantially lower.}
    \label{fig:calib-size-sweep}
\end{figure}

\begin{table*}[t!]
\centering
\caption{Activation SQNR aggregated across all layers and transformer blocks through denoising timesteps.}
\label{tab:calib_timestep}
\resizebox{0.50\textwidth}{!}{
\begin{tabular}{lccccc}
\toprule
Timestep & 0 & 5 & 10 & 15 & 19 \\
\midrule
SQNR (dB) & 33.6 & 31.2 & 29.6 & 28.8 & 27.8 \\
\bottomrule
\end{tabular}}
\end{table*}

\begin{figure}[ht]
    \centering
    \begin{subfigure}{\textwidth}
        \centering
        \includegraphics[scale=0.635]{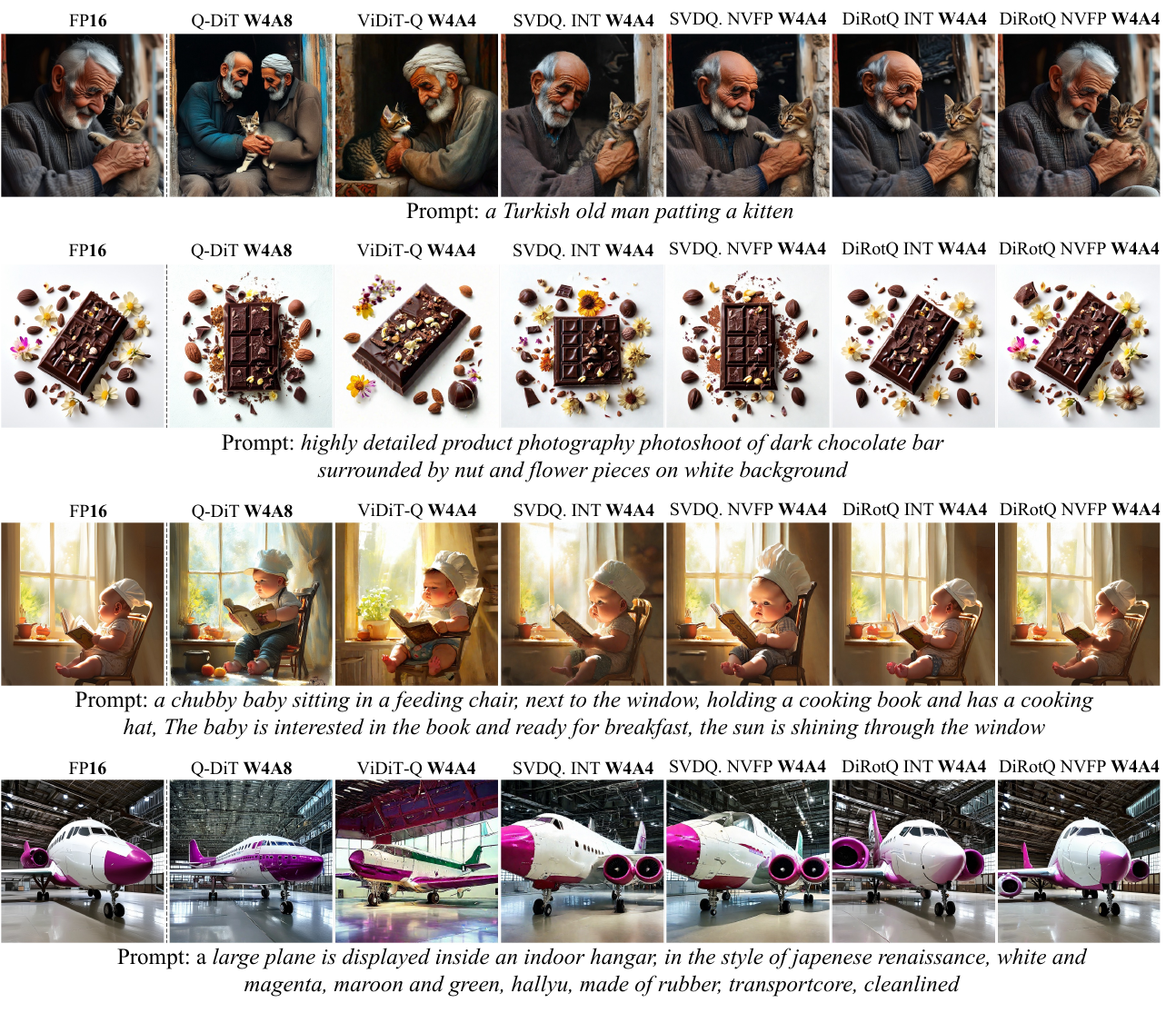}
        
        \vspace*{-3mm} \caption{Qualitative visual results of PixArt-$\Sigma$ on MJHQ-30K.}
    \end{subfigure}
    
     \vspace*{1mm} 
    \begin{subfigure}{\textwidth}
        \centering
        \includegraphics[scale=0.635]{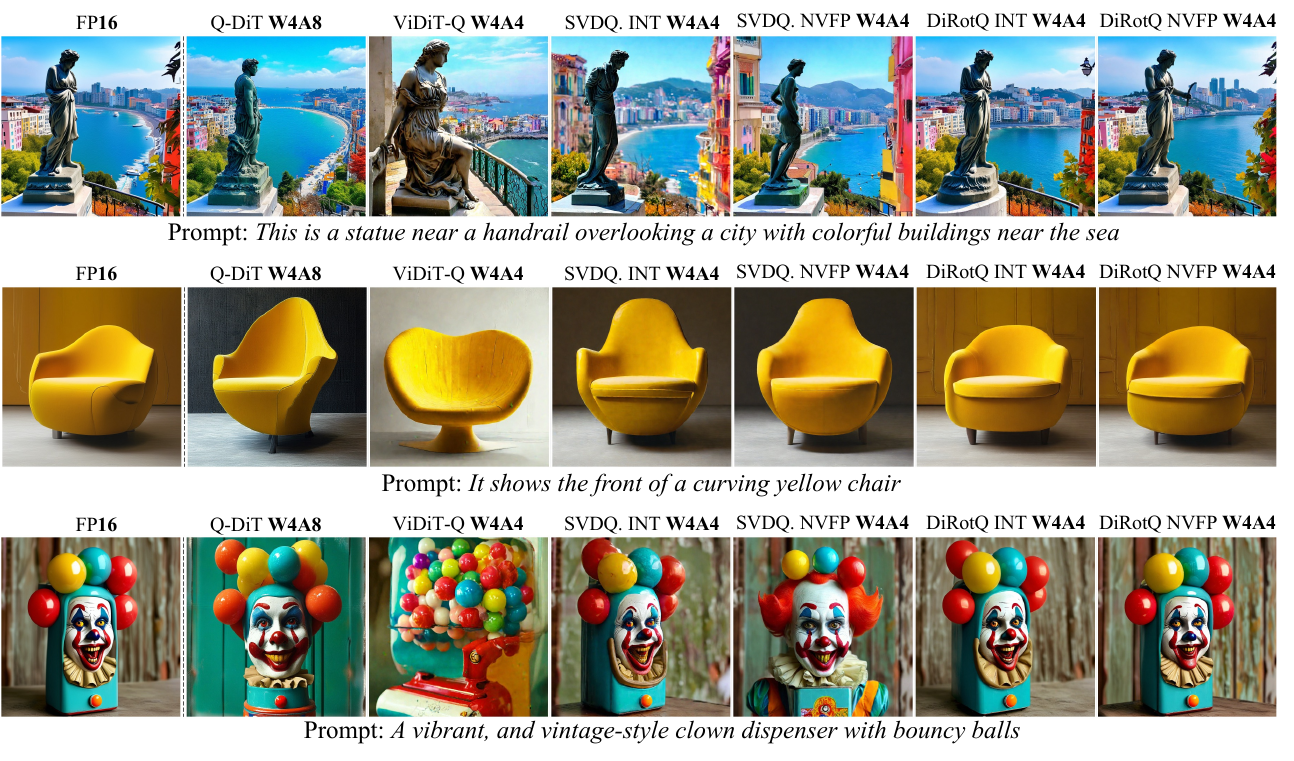}
        
        \vspace*{-2mm} \caption{Qualitative visual results of PixArt-$\Sigma$ on sDCI.}
    \end{subfigure}

     \vspace*{1mm}
    \caption{Qualitative visual results of PixArt-$\Sigma$ on the MJHQ-30K and sDCI datasets. Prompts are reproduced verbatim.}
    \label{fig:pixart_visual}
\end{figure}

\begin{figure}[ht]
    \centering
    \begin{subfigure}{\textwidth}
        \centering
        \includegraphics[scale=0.635]{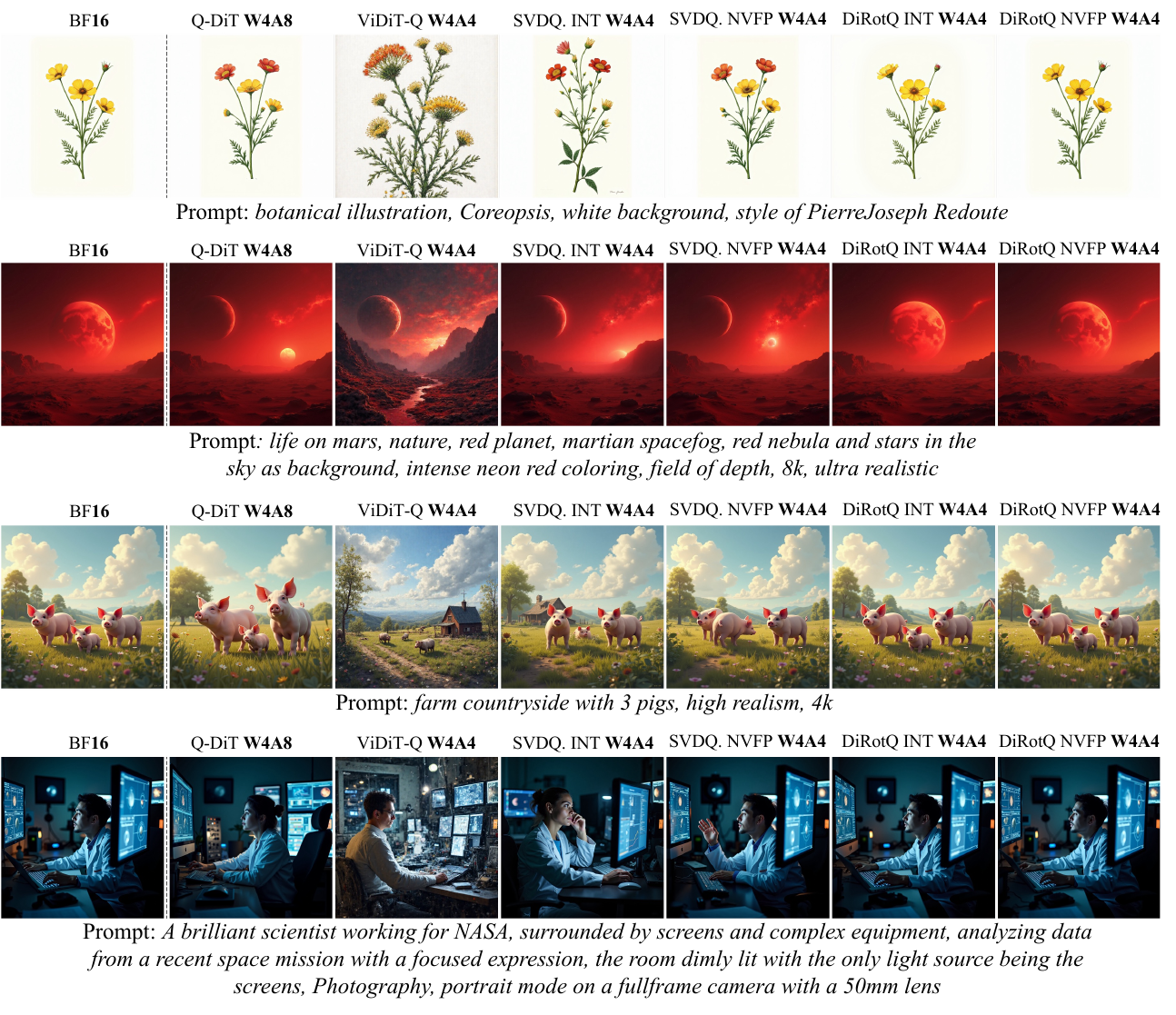}
        
        \vspace*{-1mm} \caption{Qualitative visual results of FLUX.1-dev on MJHQ-30K.}
    \end{subfigure}
    
     \vspace*{2mm} 
    \begin{subfigure}{\textwidth}
        \centering
        \includegraphics[scale=0.635]{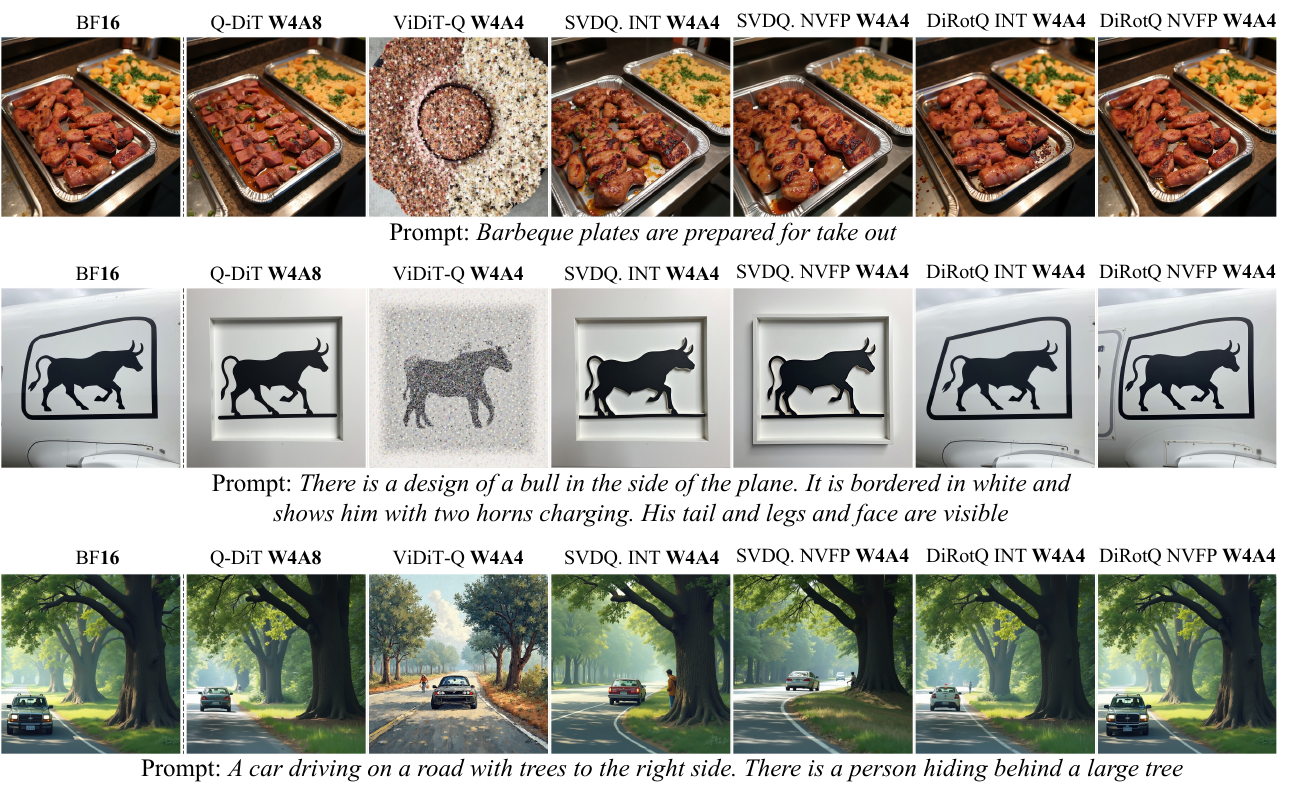}
        
        \vspace*{-1mm} \caption{Qualitative visual results of FLUX.1-dev on sDCI.}
    \end{subfigure}
     
    \caption{Qualitative visual results of FLUX.1-dev on the MJHQ-30K and sDCI datasets. Prompts are reproduced verbatim.}
    \label{fig:flux_visual}
\end{figure}

\begin{figure}[ht]
    \centering
    \begin{subfigure}{\textwidth}
        \centering
        \includegraphics[scale=0.635]{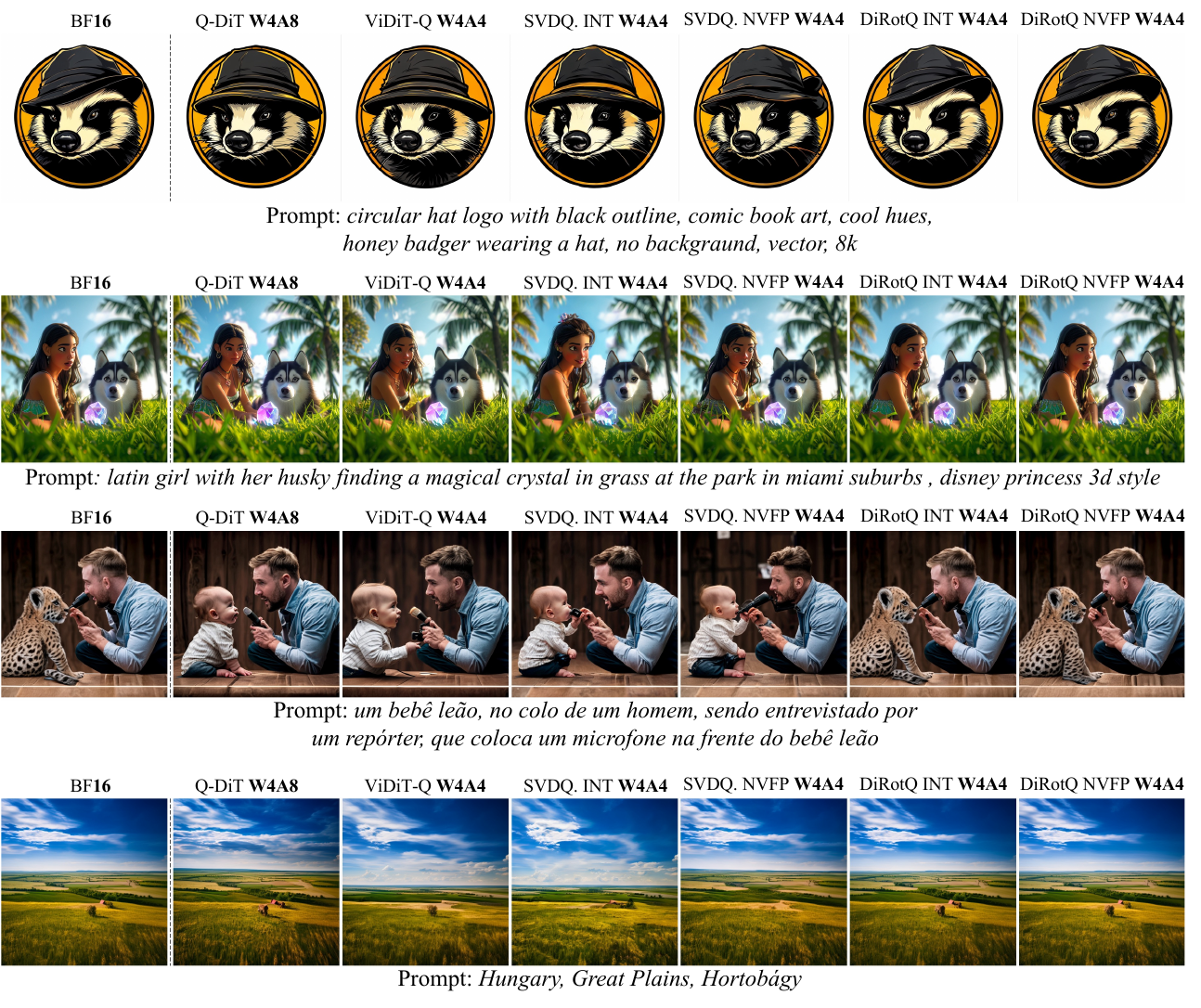}
        
        \vspace*{-1mm} \caption{Qualitative visual results of SANA-1.6B on MJHQ-30K.}
    \end{subfigure}
    
     \vspace*{2mm} 
    \begin{subfigure}{\textwidth}
        \centering
        \includegraphics[scale=0.635]{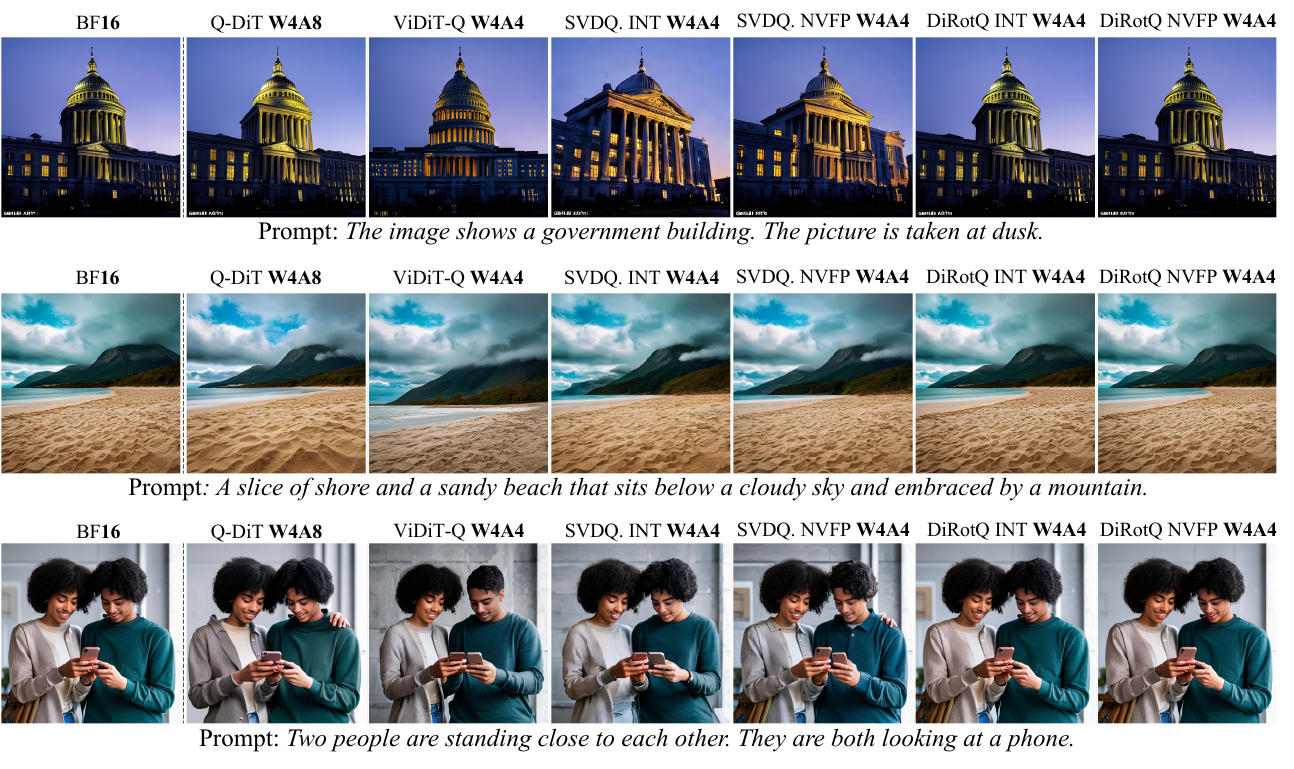}
        
        \vspace*{-1mm} \caption{Qualitative visual results of SANA-1.6B on sDCI.}
    \end{subfigure}
     
    \caption{Qualitative visual results of SANA-1.6B on the MJHQ-30K and sDCI datasets. Prompts are reproduced verbatim.}
    \label{fig:sana_visual}
\end{figure}

\newpage
\section{Licenses of existing assets}
\label{app:licenses}
Table~\ref{tab:licenses} lists each pre-trained model, dataset, and code asset used in this work, together with its license. Our use is limited to non-commercial academic research (post-training quantization, evaluation, and reporting of generated images), which is consistent with all listed licenses. We do not redistribute any model weights or datasets.

\begin{table}[ht]
\centering
\caption{Licenses of existing assets used in this work.}
\label{tab:licenses}
\small
\resizebox{0.99\textwidth}{!}{
\begin{tabular}{lll}
\toprule
\textbf{Asset} & \textbf{Type} & \textbf{License} \\
\midrule
    PixArt-$\Sigma$~\cite{chen2024pixart} & Model & CreativeML Open RAIL++-M \\
    FLUX.1-dev~\cite{flux2024} & Model & FLUX.1 [dev] Non-Commercial License \\
    SANA-1.6B~\cite{xie2025sana} & Model & NVIDIA NSCL v2-custom; Gemma ToU (text encoder) \\
    Qwen2-VL-7B-Instruct~\cite{wang2024qwen2} & Model & Apache 2.0 \\
    InternVL2-8B~\cite{chen2024expanding} & Model & MIT; Apache 2.0 (LLM component) \\
    Gemini 3.6 Flash~\citep{google2026gemini36flash} & Model & Google APIs Terms of Service; Gemini API Additional Terms of Service \\
    CLIP~\cite{radford2021learning} & Model & MIT \\
    \midrule
    COCO Captions~\cite{chen2015microsoft} & Dataset & CC-BY-4.0 (annotations); Flickr ToS (images) \\
    MJHQ-30K~\cite{li2024playground} & Dataset & Research benchmark; Midjourney ToS (images) \\
    sDCI~\cite{urbanek2024picture} & Dataset & CC-BY-NC 4.0 \\
    \midrule
    Diffusers~\cite{von-platen-etal-2022-diffusers} & Code & Apache 2.0 \\
    Triton~\cite{tillet2019triton} & Code & MIT \\
    Nunchaku~\cite{li2025svdquant} & Code & Apache 2.0 \\
    GPTQ~\cite{frantar2023gptq} & Code & Apache 2.0 \\
    Baselines~\cite{wu2024ptq4dit,chen2025q,zhao2025viditq,li2025svdquant} & Code & Apache 2.0 / MIT \\
\bottomrule
\end{tabular}}
\end{table}


\end{document}